\documentclass[twoside, 11pt]{article}
\usepackage[utf8]{inputenc}
\usepackage[english]{babel}
\usepackage{jmlr2e}
\usepackage{hyperref}
\usepackage{booktabs}
\usepackage{amssymb, amsmath}
\usepackage{enumitem}
\usepackage{ccaption}
\usepackage[normalem]{ulem}
\usepackage{cleveref}
\usepackage{subcaption}
\usepackage{stmaryrd}
\usepackage{tikz}
\usepackage{tikz-cd} 
\usetikzlibrary{arrows}
\usepackage{float}
\usepackage{xparse} 
\usepackage{pdflscape} 
\usepackage{afterpage} 
\usepackage{capt-of} 
\hypersetup{hidelinks}

\numberwithin{equation}{section}

\setlist[enumerate]{label=\normalfont(\roman*)}
\setlist{noitemsep}

\newcommand{\todo}[1]{\textcolor{blue}{TODO: #1}}

\crefname{equation}{equation}{equations}
\crefname{figure}{figure}{figures}
\crefname{lem}{lemma}{lemmas}
\crefname{ex}{example}{examples}

\tikzset{
  symbol/.style={
    draw=none,
    every to/.append style={
      edge node={node [sloped, allow upside down, auto=false]{$#1$}}}
  }
}

\def\N{{\mathbb N}}    
\def\Z{{\mathbb Z}}    
\def\R{{\mathbb R}}    
\def\C{{\mathbb C}}    

\def\V{\mathbb{V}}

\def\X{\mathbb{X}}

\def\d{\,\mathrm{d}}
\def\st{:} 
\newcommand{\eps}{\varepsilon} 
\def\1{\mathbf{1}}
\renewcommand{\phi}{\varphi}
\def\supp{\operatorname{supp}}
\newcommand{\card}[1]{|#1|}
\renewcommand{\emptyset}{\varnothing}
\renewcommand{\bar}[1]{\overline{#1}}
\newcommand{\longtoinf}[1][n\to\infty]{\underset{#1}{\longrightarrow}}
\newcommand{\closure}[1]{\overline{#1}}
\def\id{\operatorname{Id}}
\def\Ima{\operatorname{Im}}
\def\Ker{\operatorname{Ker}}

\newcommand{\Lp}[2][1]{L^{#1}(#2)}
\newcommand{\dual}[1]{{#1}^*}
\newcommand{\dualdot}[2]{#1\cdot #2}

\newcommand{\bestscore}[1]{\textcolor{black}{\textbf{#1}}}

\def\orbit{\texttt{ORBIT5K}}
\def\gudhi{\texttt{Gudhi}}
\def\mma{\texttt{MMA}}
\def\cpp{\texttt{C++}}
\def\python{\texttt{Python}}

\def\Cech{\v{C}ech}
\def\cplx{\mathcal{K}}
\def\filt{\mathcal{F}}
\def\birth{\operatorname{birth}}

\def\chains{\mathcal{C}}

\def\barcode{\mathcal{B}}
\def\diagram{\mathcal{D}}
\def\matching{M}

\newcommand{\FP}[1][\R^m]{\mathrm{FP}(#1)}
\def\sbarcode{\bar{\mathcal{B}}}
\def\cost{\mathrm{cost}}
\NewDocumentCommand{\dist}{d[]}{
	\IfNoValueTF{#1}{\def\argument{}}{\def\argument{\!\left(#1\right)}}%
    \widehat{d}_1\argument{}%
}
\def\CechCplx{\check{\mathcal{C}}}

\def\Hil{\mathrm{Hil}}
\newcommand{\betti}[1][\filt]{\beta_{#1}}
\newcommand{\ECP}[1][\filt]{\chi_{#1}}
\newcommand{\ECPn}[1][\filt]{\chi_{{#1}_n}}

\def\kernel{\kappa}				
\NewDocumentCommand{\Kernel}{d[]}{
	\IfNoValueTF{#1}{\def\argument{}}{\def\argument{\left(#1\right)}}%
    \bar{\kappa}\argument{}%
}
\NewDocumentCommand{\HT}{O{\filt} O{\kernel}}{\psi_{#1}^{#2}}
\NewDocumentCommand{\HTn}{O{\filt} O{\kernel}}{\psi_{{#1}_n}^{#2}}
\NewDocumentCommand{\HTL}{O{\filt} O{\kernel}}{\psi_{{#1}_L}^{#2}}
\NewDocumentCommand{\HTtrunc}{O{\filt} O{\kernel}}{\psi_{#1}^{{#2}, T}}
\NewDocumentCommand{\HTtruncn}{O{\filt} O{\kernel}}{\psi_{{#1}_n}^{{#2}, T}}
\def\dt{\d t}
\def\ds{\d s}

\NewDocumentCommand{\CF}{O{\R^n} O{}}{%
\operatorname{CF}_{#2}(#1)%
}

\def\cov{\operatorname{cov}}

\usepackage{lastpage}
\jmlrheading{25}{2024}{1-\pageref{LastPage}}{3/23; Revised
7/24}{7/24}{23-0353}{Olympio Hacquard and Vadim Lebovici}
\ShortHeadings{Euler Characteristic Tools for Topological Data Analysis}{Hacquard and Lebovici}

\begin{document}

\title{Euler Characteristic Tools for Topological Data Analysis}

\author{Olympio Hacquard, Vadim Lebovici}
\author{\name Olympio Hacquard \email olympio.hacquard@universite-paris-saclay.fr \\
       \addr Laboratoire de Math\'ematiques d'Orsay\\
       Université Paris-Saclay, CNRS, Inria\\
       91400 Orsay France
      \AND 
      \name Vadim Lebovici \email \email{vadim.lebovici@maths.ox.ac.uk} \\
       \addr Mathematical Institute\\
        University of Oxford,\\
       Oxford, United Kingdom
}

%

\editor{Sayan Mukherjee}

\maketitle 

\vspace{-1cm}

\begin{abstract}%
In this article, we study Euler characteristic techniques in topological data analysis. Pointwise computing the Euler characteristic of a family of simplicial complexes built from data gives rise to the so-called Euler characteristic profile. We show that this simple descriptor achieves state-of-the-art performance in supervised tasks at a meager computational cost. Inspired by signal analysis, we compute hybrid transforms of Euler characteristic profiles. These integral transforms mix Euler characteristic techniques with Lebesgue integration to provide highly efficient compressors of topological signals. As a consequence, they show remarkable performances in unsupervised settings. On the qualitative side, we provide numerous heuristics on the topological and geometric information captured by Euler profiles and their hybrid transforms. Finally, we prove stability results for these descriptors as well as asymptotic guarantees in random settings.
\end{abstract}



\section{Introduction}
Extracting topological information from data of various natures follows a machinery that finds its origins in the works of \citet{edelsbrunner2000topological}. The main idea consists of building a one-parameter family of topological spaces on top of data and tracking the evolution of its topology, typically via homological computations. This multi-scale topological information is recorded in the form of what is called a \emph{persistence diagram}; see~\citet{edelsbrunner2000topological, edelsbrunner2022computational}. The space of persistence diagrams is a metric space for the so-called \emph{bottleneck distance}, \citep{cohen2007stability}, but it cannot be isometrically embedded into a Hilbert space \citep{carriere2018,bubenik2020embeddings}. At the cost of losing some information, these diagrams are still often turned into vectors to perform various learning tasks such as classification, clustering, or regression. Most commonly used techniques include persistence images \citep{adams2017persistence}, landscapes \citep{bubenik2015statistical}, and more recently measure-oriented vectorizations in \citet{royer2021atol} and neural network methods from \citet{carriere2020perslay, reinauer2021persformer}. An overview of topological methods in machine learning has been presented in the survey of \citet{hensel2021survey}. These methods have demonstrated their efficiency in a wide variety of applications and types of data, such as health applications \citep{rieck2020uncovering, fernandez2022topological, aukerman2021persistent}, biology \citep{ichinomiya2020protein, rabadan2019topological} or material sciences \citep{lee2017quantifying, hiraoka2016hierarchical}.

In many practical scenarios, it is natural to look at data with more than one parameter, i.e., to consider multi-parameter families of topological spaces instead of one-parameter ones. It allows one to cope with outliers by filtering the space with respect to an estimated local density, or to deal with intrinsically multi-parameter data, such as blood cells with several biomarkers. However, there does not exist a complete combinatorial descriptor similar to the persistence diagram that could make them usable in practice \citep{Carlsson2009}. One of the main objectives of this field is to build informative descriptors of such families. Although not intrinsically multi-parameter, persistence landscapes have successfully been generalized to the multi-parameter setting in \citet{vipondlandscapes} and persistence images to the two-parameter setting in~\citet{carriere2020multiparameter}. Besides their high level of sophistication, the main limitation of these tools is their computational cost; see \citet[Table~2]{carriere2020multiparameter} and \Cref{sec:timing}. 

In contrast, some topological methods do not compute homological information---thus bypassing the computation of persistence diagrams---but rather compute the Euler characteristic of the topological spaces at hand. The Euler characteristic of a simplicial complex is a celebrated topological invariant that is simply the alternated sum of the number of simplices of each dimension. Considering the pointwise Euler characteristic of a one-parameter family of simplicial complexes gives rise to a functional multi-scale descriptor called the \textit{Euler characteristic curve}. 

Though Euler characteristic-based descriptors may appear coarse, we highlight four main reasons to use them. First, they have demonstrated a good predictive power in various settings \citep{worsley1992three,richardson2014efficient, smith2021euler, jiang2020weighted, amezquita2022measuring}. Second, the simplicity of these descriptors translates into a reduced computational cost. They can be computed in linear time in the number of simplices in a simplicial filtration instead of typically matrix multiplication time for persistence diagrams~\citep{milosavljevic2011zigzag}. Moreover, the locality of the Euler characteristic can be exploited to design highly efficient algorithms computing Euler curves, as in~\citet{heiss2017streaming}. Third, there are several known theoretical results on the Euler characteristic of a random complex. Mean formulae for the Euler characteristic of superlevel sets of random fields are proven in \citet{AT10}, and results for the Euler characteristic of a complex built on a Poisson process are established in Corollary 4.2 of \citet{bobrowski2014distance} and Corollary 6.2 of \citet{bobrowski2017vanishing}. Furthermore, Euler curves associated with random point clouds are proven to be asymptotically normal for a well-chosen sampling regime in \citet{krebs2021approximation}, where the authors also apply this construction to bootstrap.
Fourth, they naturally generalise to the multi-parameter setting, becoming so-called \emph{Euler characteristic surfaces} \citep{beltramo2022euler} and \emph{profiles} \citep{DG22}. 

We demonstrate that these tools reach state-of-the-art performance at a minimal computational cost when coupled with a powerful classifier such as a gradient boosting or a random forest. However, due to their simplicity, these descriptors do not manage to linearly separate the different classes or be competitive on unsupervised tasks. Inspired by signal analysis, we cope with these limitations by studying integral transforms of Euler characteristic curves and profiles. More precisely, we consider a general notion of integral transforms mixing Lebesgue integration and Euler characteristic techniques recently introduced in \citet{L22} under the name of \emph{hybrid transforms}. In the one-parameter case, hybrid transforms are classical integral transforms of Euler curves. Similarly, hybrid transforms depend on a choice of kernel which offers a wide variety of possible signal decompositions. Yet, hybrid transforms differ from classical integral transforms in general. In so doing, they enjoy many specific appealing properties, such as compatibility with topological operations from Euler calculus \cite[Section~5]{L22}. Most importantly, in the context of multi-parameter sublevel-sets persistence, hybrid transforms can be expressed as one-parameter hybrid transforms of Euler curves associated with a linear combination of the filtration functions. Consequently, mean formulae for hybrid transforms associated with Gaussian random fields are derived in \cite[Section~8]{L22}. Studying the asymptotic behaviour of topological descriptors of random complexes is a deeply-studied question in the one-parameter setting; see \citet{bobrowski2018topology} for a survey. Together with the works of \citet{botnan2022consistency}, our results form the first occurrence of limiting theorems in a multi-persistence framework in the literature.  

\paragraph{Contributions and outline.} In this article, we show that Euler characteristic profiles and their hybrid transforms are informative and highly efficient topological descriptors. Throughout the paper, we use classical methods based on persistence diagrams as a baseline for our descriptors. After introducing the necessary notions in \Cref{sec:definitions}, we provide heuristics on how to choose the kernel of hybrid transforms and give many examples of the type of topological and geometric behaviour Euler curves and their integral transforms can capture from data in \Cref{sec:method}. Most importantly, our main contributions are the following:
\begin{itemize}
    \item We demonstrate that Euler profiles achieve state-of-the-art accuracy in supervised classification and regression tasks when coupled with a random forest or a gradient boosting (\Cref{sec:curv-reg,sec:orbit,sec:graph}) at a very low computational cost (\Cref{sec:timing}). Note that the multi-parameter nature of our tools and their computational simplicity allows us to use up to $5$-parameter filtrations to classify graph data. They typically outperform persistence diagrams-based vectorizations, both in terms of accuracy and computational time.
    \item We demonstrate that hybrid transforms act as highly efficient information compressors and typically require a much smaller resolution than Euler profiles to reach a similar performance. They can also outperform Euler profiles in unsupervised classification tasks and in supervised tasks when plugging a linear classifier (\Cref{fig:hidden_illus} and \Cref{sec:curv-reg,sec:orbit,sec:sidney}). In \Cref{sec:sidney}, we illustrate their ability to capture fine-grained information on a real-world data set.
    \item We provide several theoretical guarantees for these descriptors. First, we prove stability properties that clarify the robustness of our tools with respect to perturbations (\Cref{sec:stability}). Expressed in terms of $L_1$ norms, these are also hints of the sensitivity of our tools to the underlying geometry of the data at hand. Similarly to persistence diagrams, we can establish the pointwise convergence of hybrid transforms associated with random samples and their asymptotic normality for a specific filtration function. We also establish a law of large numbers in a multi-filtration set-up (\Cref{sec:stat}). 
\end{itemize}
Finally, \Cref{sec:proofs} is devoted to the proofs of the results stated in \Cref{sec:stability,sec:stat}.

\section{Topological descriptors}\label{sec:definitions}
This section presents all the necessary notions from simplicial geometry and the construction of the topological descriptors used throughout the article. Let us first introduce some conventions.
\begin{enumerate}
    \item The dual of a vector space $\V$ is denoted by $\dual{\V}$, and $\R^m$ will always be identified with its dual under the canonical isomorphism. For $\xi\in\dual{\R^m}$ and $t\in\R^m$, we often denote $\dualdot{\xi}{t} = \xi(t)$.
    \item We denote by~$\dual{\R_+^m}$ the cone of linear forms on $\R^m$ that are non-decreasing with respect to the coordinatewise order on~$\R^m$, or equivalently that have non-negative canonical coordinates. 
    \item Let~$I$ be an interval of~$\R$ and denote by~$\Lp{I}$ the space of absolutely integrable complex-valued functions on~$I$.
    \item Let $p\in[1,\infty]$ and let $f:\R^m\to\C$ be locally $p$-integrable. We denote by $\|f\|_{p,M}$ the $p$-norm of $f\cdot \1_{[-M,M]^m}$. If $f$ is $p$-integrable, we denote its $p$-norm by~$\|f\|_p$.
    \item We always consider the coordinate-wise order on $\R^m$.
\end{enumerate}

\subsection{Simplicial complexes, filtrations}\label{sec:cplx-filt}
A \emph{(finite) abstract simplicial complex} $\cplx$, or simply \emph{simplicial complex}, is a finite collection of finite sets that is closed under taking subsets. An element $\sigma\in\cplx$ is called a \emph{simplex}, and subsets of $\sigma$ are called \emph{faces} of $\sigma$. The inclusion between simplices induces a partial order on $\cplx$ that we denote simply by $\leq$. The \emph{dimension} of a simplex with $k$ elements is equal to $k-1$. The \emph{Euler characteristic} of a simplicial complex~$\cplx$~is the integer:
\begin{equation*}
    \chi(\cplx) = \sum_{\sigma\in\cplx}(-1)^{\dim\sigma}.
\end{equation*}
Until the end of this section, we let $\cplx$ be a finite simplicial complex. An \emph{$m$-parameter filtration} of $\cplx$ is a family $\filt = \left(\filt_t\right)_{t\in\R^m}$ of subcomplexes $\filt_t\subseteq \cplx$ that is increasing with respect to inclusions, i.e., such that $\filt_t\subseteq\filt_{t'}$ for any $t, t'\in\R^m$ with $t\leq t'$. From now on, we do not refer explicitly to $\cplx$ when it is clear from the context. Many filtrations can be introduced by considering sublevel sets of functions:

\begin{example}
\label{ex:sublvl}
    Let $f:\cplx\to\R^m$ be a non-decreasing map for the inclusion of simplices, i.e., such that $f(\sigma) \leq f(\tau)$ for any $\sigma \leq \tau \in \cplx$. The map $f$ induces an $m$-parameter filtration of $\cplx$ called \emph{sublevel-sets filtration}, denoted by $\filt_f$, and formed by the subcomplexes $(\filt_f)_t = \{f \leq t\} := \{\sigma \in\cplx \st f(\sigma) \leq t\}$ for any $t\in\R^m$. We sometimes refer to the function $f$ as the \emph{filter} of~$\filt_f$. 
\end{example}

A popular example of simplicial complex is the \Cech{} complex of a \emph{point cloud}, that is, a finite subset of $\mathbb{R}^d$. This complex captures a lot of information on the geometry of the point cloud.

\begin{example}\label{ex:Cech}
    Let $\X\subseteq \R^d$ be finite. The \emph{\Cech{} complex at scale~$t\geq 0$} is the simplicial complex $\CechCplx(\X, t)$ defined such that for $(x_0, \ldots, x_{k}) \in \X^{k+1}$, the simplex~$\{x_0, \ldots, x_k\}$ is in $\CechCplx(\X, t)$ if the intersection of closed balls $\cap_{l=0}^{k} \closure{B}(x_l, t)$ is non-empty.  The \emph{\Cech{} filtration}, is defined at each $t\in\R$ as the \Cech{} complex at scale~$t$ for~$t\geq 0$, and as the empty set for $t<0$. For computational reasons, we rather use a homotopy equivalent complex in numerical experiments, called the \emph{alpha filtration}, which is a subcomplex of the Delaunay triangulation; see~\citet{bauer2017morse}. See \Cref{fig:Cech-filtration} for an illustration.
\end{example}

The properties of the \Cech{} complex of a random point cloud have been deeply studied theoretically. We refer to \citet{bobrowski2018topology} and \citet{owada2022convergence} for the most recent results. When doing multi-parameter persistence, a common technique is to couple the \Cech{} complex with some function on the data. Typically, we cope with outliers by coupling a \Cech{} filtration with a density estimator built from the data at hand. This falls under the framework of function-\Cech{} filtrations:

\begin{example}\label{ex:function-Cech}
    Let $\X\subseteq \R^d$ be finite and $f=(f_1,\ldots,f_m):\X\to\R^m$ be a bounded function. The \emph{function-\Cech{} filtration} is the $(m+1)$-parameter filtration~$\CechCplx(\X,f)$ of $2^\X$ defined for $r\in\R$ and $t=(t_1, \ldots,t_m)\in\R^{m}$ by:
    \begin{equation*}
        \CechCplx(\X,f)_{(r,t)} = \left\{ \sigma \in \CechCplx(\X,r) \st \max_{x\in\sigma} f_i(x) \leq t_i, \, \forall 1\leq i \leq m\right\}.
    \end{equation*}
    Again, we rather use \emph{function-alpha filtration} in numerical experiments, which are defined similarly using alpha complexes.
\end{example}

Let~$\filt$ be an $m$-parameter filtration and $\sigma \in \cplx$. The \emph{support of $\sigma$} is the set $\supp(\sigma):= \{t\in\R^m \st \sigma \in\filt_t\}$. A filtration is called \emph{finitely generated} if the support of any simplex appearing in the filtration is either empty or has a finite number of minimal elements; see \Cref{fig:finitely-gen-filt} for an illustration. Moreover, if the support of any simplex has at most one minimal element, then the filtration is called \emph{one-critical}. In that case, one denotes by $t(\sigma)$ the minimal element of $\supp(\sigma)$. For instance, function-Cech and function-alpha filtrations are one-critical. On the contrary, the degree-Rips bifiltration is not \citep{lesnick2016interactive}. Note that sublevel-sets filtrations are one-critical. Conversely, any one-critical filtration is a sublevel set filtration for the function $f : \sigma\in\cplx\mapsto t(\sigma)$.

\subsection{Persistence diagrams}\label{sec:pers-hom}
Given a filtration of a simplicial complex, we want to extract multi-scale topological information from data. This is the objective of \textit{persistent homology}, which constitutes the main tool of topological data analysis. It has found many practical applications \citep{rieck2020uncovering, fernandez2022topological, aukerman2021persistent, ichinomiya2020protein, rabadan2019topological, lee2017quantifying, hiraoka2016hierarchical} as well as applications to other fields of theoretical mathematics, such as symplectic geometry \citep{polterovich2020topological}. This section introduces the basic objects of persistent homology as introduced in classical textbooks; see \citet{edelsbrunner2022computational,oudot2017persistence}. We try to keep the notions as intuitive as possible and do not lay out the technical details of homology theory. 

The central tool of persistence theory is \emph{homology}. Intuitively, given a topological space $X$, the $k$-th homology of $X$ (with coefficient in a field) is a vector space whose dimension is equal to the number of independent $k$-dimensional holes of $X$. By $0$-dimensional (resp. $1$-dimensional, $2$-dimensional) holes, we mean connected components (resp. cycles, voids). These $k$-dimensional holes are often called \emph{homological features} of~$X$. One can also define the homology of a simplicial complex $\cplx$, denoted by $H_k(\cplx)$ for each integer~$k\geq 0$, in such a way that they coincide with the above intuition when looking at the geometric realisation of the simplicial complex $\cplx$. 

Given a one-parameter filtration $\filt$ of a simplicial complex~$\cplx$, one of the main properties of homology implies that for any $t\leq t'$ in $\R$, the inclusion of complexes $\filt_t \subseteq \filt_{t'}$ induces a linear map $H_k(\filt_t) \to H_k(\filt_{t'})$. The idea of \emph{persistent homology} is to keep track of homological features appearing in the filtration through these maps. Each generator appears at some $a\in\R$ called its \emph{birth} and disappears at some $b>a$ called its \emph{death}. The couple $[a,b)$ is called the \emph{bar} of the corresponding homological feature. 
The multiset of bars~$[a,b)$ for each homological feature appearing in the filtration is called the \emph{degree $k$ persistence barcode} of~$\filt$. One can also represent this barcode as a multiset of points $(a,b)\in\R^2$ called the \emph{degree $k$ persistence diagram} of $\filt$. We give an example of the construction of the persistence diagram of the \Cech{} filtration in \Cref{fig:PD_cech}. In this case, the persistence diagram gives a lot of information on the topology and the geometry of the underlying point cloud. Here, when the radius of the balls is smaller than the smallest distance between any two points, we have as many connected components as points. As the radii of the balls grow, connected components of the union of balls merge (or die) one by one, except for one that never dies. Therefore, the degree 0 persistence diagram has only points born at 0. As for the degree 1 persistence diagram, a cycle appears when the radius of the balls is large enough and is filled approximately at the radius of the underlying circle, hence a single point in the persistence diagram.

\begin{figure}
	\centering
	\begin{subfigure}[t]{0.22\linewidth}
        \centering
		\scalebox{.2}{
		\definecolor{wwccff}{rgb}{0.4,0.8,1}
		\begin{tikzpicture}[line cap=round,line join=round,>=triangle 45,x=1cm,y=1cm]
		\node (x1) at (1.87236,2.74073) {};
		\node (x2) at (0.4853659626218362,0.37468619425046223) {};
		\node (x3) at (0.41216450288947193,-1.896847399280642) {};
		\node (x4) at (4.972687388202764,4.018935109570714) {};
		\node (x5) at (5.897347560746713,3.257450261593348) {};
		\node (x6) at (8.154606217251057,1.462521691360985) {};
		\node (x7) at (7.0939666075682934,2.3327900890494035) {};
		\node (x8) at (0.86611,1.92485) {};
		\node (x9) at (8.208997992106584,0.8098203930946715) {};
		\node (x10) at (0.8201028143059201,-0.6458365776035415) {};
		\node (x11) at (1.5,3.08) {};
		\node (x12) at (1.1564795986048635,-2.4536860982368975) {};
		\node (x13) at (1.953943883721049,-3.051995621647685) {};
		\node (x14) at (2.5334437222550013,-3.4742088700909) {};
		\node (x15) at (4.7091147164760585,-3.637384194657479) {};
		\node (x16) at (6.359677647018687,-2.7800367473700542) {};
		\node (x17) at (7.319919909541324,-2.4679610352636674) {};
		\node (x18) at (8.108600644946456,-1.3801255381531434) {};
		\node (x19) at (5.108600644946456,-5.3801255381531434) {};
		\def\rad{0.4cm}
		\draw [line width=1.2pt,color=wwccff,fill=wwccff,fill opacity=0.4] (x1) circle (\rad);
		\draw [line width=1.2pt,color=wwccff,fill=wwccff,fill opacity=0.4] (x2) circle (\rad);
		\draw [line width=1.2pt,color=wwccff,fill=wwccff,fill opacity=0.4] (x3) circle (\rad);
		\draw [line width=1.2pt,color=wwccff,fill=wwccff,fill opacity=0.4] (x4) circle (\rad);
		\draw [line width=1.2pt,color=wwccff,fill=wwccff,fill opacity=0.4] (x5) circle (\rad);
		\draw [line width=1.2pt,color=wwccff,fill=wwccff,fill opacity=0.4] (x6) circle (\rad);
		\draw [line width=1.2pt,color=wwccff,fill=wwccff,fill opacity=0.4] (x7) circle (\rad);
		\draw [line width=1.2pt,color=wwccff,fill=wwccff,fill opacity=0.4] (x8) circle (\rad);
		\draw [line width=1.2pt,color=wwccff,fill=wwccff,fill opacity=0.4] (x9) circle (\rad);
		\draw [line width=1.2pt,color=wwccff,fill=wwccff,fill opacity=0.4] (x10) circle (\rad);
		\draw [line width=1.2pt,color=wwccff,fill=wwccff,fill opacity=0.4] (x11) circle (\rad);
		\draw [line width=1.2pt,color=wwccff,fill=wwccff,fill opacity=0.4] (x12) circle (\rad);
		\draw [line width=1.2pt,color=wwccff,fill=wwccff,fill opacity=0.4] (x13) circle (\rad);
		\draw [line width=1.2pt,color=wwccff,fill=wwccff,fill opacity=0.4] (x14) circle (\rad);
		\draw [line width=1.2pt,color=wwccff,fill=wwccff,fill opacity=0.4] (x15) circle (\rad);
		\draw [line width=1.2pt,color=wwccff,fill=wwccff,fill opacity=0.4] (x16) circle (\rad);
		\draw [line width=1.2pt,color=wwccff,fill=wwccff,fill opacity=0.4] (x17) circle (\rad);
		\draw [line width=1.2pt,color=wwccff,fill=wwccff,fill opacity=0.4] (x18) circle (\rad);
		\begin{scriptsize}
			\draw [fill=wwccff] (x1) circle (3pt);
			\draw [fill=wwccff] (x2) circle (3pt);
			\draw [fill=wwccff] (x3) circle (3pt);
			\draw [fill=wwccff] (x4) circle (3pt);
			\draw [fill=wwccff] (x5) circle (3pt);
			\draw [fill=wwccff] (x6) circle (3pt);
			\draw [fill=wwccff] (x7) circle (3pt);
			\draw [fill=wwccff] (x8) circle (3pt);
			\draw [fill=wwccff] (x9) circle (3pt);
			\draw [fill=wwccff] (x10) circle (3pt);
			\draw [fill=wwccff] (x11) circle (3pt);
			\draw [fill=wwccff] (x12) circle (3pt);
			\draw [fill=wwccff] (x13) circle (3pt);
			\draw [fill=wwccff] (x14) circle (3pt);
			\draw [fill=wwccff] (x15) circle (3pt);
			\draw [fill=wwccff] (x16) circle (3pt);
			\draw [fill=wwccff] (x17) circle (3pt);
			\draw [fill=wwccff] (x18) circle (3pt);
		\end{scriptsize}
		\end{tikzpicture}}
        \caption{$t = 0.4$}
    \end{subfigure}
	\begin{subfigure}[t]{0.22\linewidth}
        \centering
		\scalebox{.2}{
		\definecolor{wwccff}{rgb}{0.4,0.8,1}
		\begin{tikzpicture}[line cap=round,line join=round,>=triangle 45,x=1cm,y=1cm]
		\node (x1) at (1.87236,2.74073) {};
		\node (x2) at (0.4853659626218362,0.37468619425046223) {};
		\node (x3) at (0.41216450288947193,-1.896847399280642) {};
		\node (x4) at (4.972687388202764,4.018935109570714) {};
		\node (x5) at (5.897347560746713,3.257450261593348) {};
		\node (x6) at (8.154606217251057,1.462521691360985) {};
		\node (x7) at (7.0939666075682934,2.3327900890494035) {};
		\node (x8) at (0.86611,1.92485) {};
		\node (x9) at (8.208997992106584,0.8098203930946715) {};
		\node (x10) at (0.8201028143059201,-0.6458365776035415) {};
		\node (x11) at (1.5,3.08) {};
		\node (x12) at (1.1564795986048635,-2.4536860982368975) {};
		\node (x13) at (1.953943883721049,-3.051995621647685) {};
		\node (x14) at (2.5334437222550013,-3.4742088700909) {};
		\node (x15) at (4.7091147164760585,-3.637384194657479) {};
		\node (x16) at (6.359677647018687,-2.7800367473700542) {};
		\node (x17) at (7.319919909541324,-2.4679610352636674) {};
		\node (x18) at (8.108600644946456,-1.3801255381531434) {};
		\node (x19) at (5.108600644946456,-5.3801255381531434) {};
		\def\rad{1cm}
		\draw [line width=1.2pt,color=wwccff,fill=wwccff,fill opacity=0.4] (x1) circle (\rad);
		\draw [line width=1.2pt,color=wwccff,fill=wwccff,fill opacity=0.4] (x2) circle (\rad);
		\draw [line width=1.2pt,color=wwccff,fill=wwccff,fill opacity=0.4] (x3) circle (\rad);
		\draw [line width=1.2pt,color=wwccff,fill=wwccff,fill opacity=0.4] (x4) circle (\rad);
		\draw [line width=1.2pt,color=wwccff,fill=wwccff,fill opacity=0.4] (x5) circle (\rad);
		\draw [line width=1.2pt,color=wwccff,fill=wwccff,fill opacity=0.4] (x6) circle (\rad);
		\draw [line width=1.2pt,color=wwccff,fill=wwccff,fill opacity=0.4] (x7) circle (\rad);
		\draw [line width=1.2pt,color=wwccff,fill=wwccff,fill opacity=0.4] (x8) circle (\rad);
		\draw [line width=1.2pt,color=wwccff,fill=wwccff,fill opacity=0.4] (x9) circle (\rad);
		\draw [line width=1.2pt,color=wwccff,fill=wwccff,fill opacity=0.4] (x10) circle (\rad);
		\draw [line width=1.2pt,color=wwccff,fill=wwccff,fill opacity=0.4] (x11) circle (\rad);
		\draw [line width=1.2pt,color=wwccff,fill=wwccff,fill opacity=0.4] (x12) circle (\rad);
		\draw [line width=1.2pt,color=wwccff,fill=wwccff,fill opacity=0.4] (x13) circle (\rad);
		\draw [line width=1.2pt,color=wwccff,fill=wwccff,fill opacity=0.4] (x14) circle (\rad);
		\draw [line width=1.2pt,color=wwccff,fill=wwccff,fill opacity=0.4] (x15) circle (\rad);
		\draw [line width=1.2pt,color=wwccff,fill=wwccff,fill opacity=0.4] (x16) circle (\rad);
		\draw [line width=1.2pt,color=wwccff,fill=wwccff,fill opacity=0.4] (x17) circle (\rad);
		\draw [line width=1.2pt,color=wwccff,fill=wwccff,fill opacity=0.4] (x18) circle (\rad);
		\begin{scriptsize}
			\draw [fill=wwccff] (x1) circle (3pt);
			\draw [fill=wwccff] (x2) circle (3pt);
			\draw [fill=wwccff] (x3) circle (3pt);
			\draw [fill=wwccff] (x4) circle (3pt);
			\draw [fill=wwccff] (x5) circle (3pt);
			\draw [fill=wwccff] (x6) circle (3pt);
			\draw [fill=wwccff] (x7) circle (3pt);
			\draw [fill=wwccff] (x8) circle (3pt);
			\draw [fill=wwccff] (x9) circle (3pt);
			\draw [fill=wwccff] (x10) circle (3pt);
			\draw [fill=wwccff] (x11) circle (3pt);
			\draw [fill=wwccff] (x12) circle (3pt);
			\draw [fill=wwccff] (x13) circle (3pt);
			\draw [fill=wwccff] (x14) circle (3pt);
			\draw [fill=wwccff] (x15) circle (3pt);
			\draw [fill=wwccff] (x16) circle (3pt);
			\draw [fill=wwccff] (x17) circle (3pt);
			\draw [fill=wwccff] (x18) circle (3pt);
		\end{scriptsize}
		\end{tikzpicture}}
        \caption{$t = 1.5$}
    \end{subfigure}
	\begin{subfigure}[t]{0.22\linewidth}
        \centering
		\scalebox{.2}{
		\definecolor{wwccff}{rgb}{0.4,0.8,1}
		\begin{tikzpicture}[line cap=round,line join=round,>=triangle 45,x=1cm,y=1cm]
		\node (x1) at (1.87236,2.74073) {};
		\node (x2) at (0.4853659626218362,0.37468619425046223) {};
		\node (x3) at (0.41216450288947193,-1.896847399280642) {};
		\node (x4) at (4.972687388202764,4.018935109570714) {};
		\node (x5) at (5.897347560746713,3.257450261593348) {};
		\node (x6) at (8.154606217251057,1.462521691360985) {};
		\node (x7) at (7.0939666075682934,2.3327900890494035) {};
		\node (x8) at (0.86611,1.92485) {};
		\node (x9) at (8.208997992106584,0.8098203930946715) {};
		\node (x10) at (0.8201028143059201,-0.6458365776035415) {};
		\node (x11) at (1.5,3.08) {};
		\node (x12) at (1.1564795986048635,-2.4536860982368975) {};
		\node (x13) at (1.953943883721049,-3.051995621647685) {};
		\node (x14) at (2.5334437222550013,-3.4742088700909) {};
		\node (x15) at (4.7091147164760585,-3.637384194657479) {};
		\node (x16) at (6.359677647018687,-2.7800367473700542) {};
		\node (x17) at (7.319919909541324,-2.4679610352636674) {};
		\node (x18) at (8.108600644946456,-1.3801255381531434) {};
		\node (x19) at (5.108600644946456,-5.3801255381531434) {};
		\def\rad{1.5cm}
		\draw [line width=1.2pt,color=wwccff,fill=wwccff,fill opacity=0.4] (x1) circle (\rad);
		\draw [line width=1.2pt,color=wwccff,fill=wwccff,fill opacity=0.4] (x2) circle (\rad);
		\draw [line width=1.2pt,color=wwccff,fill=wwccff,fill opacity=0.4] (x3) circle (\rad);
		\draw [line width=1.2pt,color=wwccff,fill=wwccff,fill opacity=0.4] (x4) circle (\rad);
		\draw [line width=1.2pt,color=wwccff,fill=wwccff,fill opacity=0.4] (x5) circle (\rad);
		\draw [line width=1.2pt,color=wwccff,fill=wwccff,fill opacity=0.4] (x6) circle (\rad);
		\draw [line width=1.2pt,color=wwccff,fill=wwccff,fill opacity=0.4] (x7) circle (\rad);
		\draw [line width=1.2pt,color=wwccff,fill=wwccff,fill opacity=0.4] (x8) circle (\rad);
		\draw [line width=1.2pt,color=wwccff,fill=wwccff,fill opacity=0.4] (x9) circle (\rad);
		\draw [line width=1.2pt,color=wwccff,fill=wwccff,fill opacity=0.4] (x10) circle (\rad);
		\draw [line width=1.2pt,color=wwccff,fill=wwccff,fill opacity=0.4] (x11) circle (\rad);
		\draw [line width=1.2pt,color=wwccff,fill=wwccff,fill opacity=0.4] (x12) circle (\rad);
		\draw [line width=1.2pt,color=wwccff,fill=wwccff,fill opacity=0.4] (x13) circle (\rad);
		\draw [line width=1.2pt,color=wwccff,fill=wwccff,fill opacity=0.4] (x14) circle (\rad);
		\draw [line width=1.2pt,color=wwccff,fill=wwccff,fill opacity=0.4] (x15) circle (\rad);
		\draw [line width=1.2pt,color=wwccff,fill=wwccff,fill opacity=0.4] (x16) circle (\rad);
		\draw [line width=1.2pt,color=wwccff,fill=wwccff,fill opacity=0.4] (x17) circle (\rad);
		\draw [line width=1.2pt,color=wwccff,fill=wwccff,fill opacity=0.4] (x18) circle (\rad);
		\begin{scriptsize}
			\draw [fill=wwccff] (x1) circle (3pt);
			\draw [fill=wwccff] (x2) circle (3pt);
			\draw [fill=wwccff] (x3) circle (3pt);
			\draw [fill=wwccff] (x4) circle (3pt);
			\draw [fill=wwccff] (x5) circle (3pt);
			\draw [fill=wwccff] (x6) circle (3pt);
			\draw [fill=wwccff] (x7) circle (3pt);
			\draw [fill=wwccff] (x8) circle (3pt);
			\draw [fill=wwccff] (x9) circle (3pt);
			\draw [fill=wwccff] (x10) circle (3pt);
			\draw [fill=wwccff] (x11) circle (3pt);
			\draw [fill=wwccff] (x12) circle (3pt);
			\draw [fill=wwccff] (x13) circle (3pt);
			\draw [fill=wwccff] (x14) circle (3pt);
			\draw [fill=wwccff] (x15) circle (3pt);
			\draw [fill=wwccff] (x16) circle (3pt);
			\draw [fill=wwccff] (x17) circle (3pt);
			\draw [fill=wwccff] (x18) circle (3pt);
		\end{scriptsize}
		\end{tikzpicture}}
        \caption{$t = 2.1$}
    \end{subfigure}
    \begin{subfigure}[t]{0.26\linewidth}
    \centering
    \includegraphics[height=3.1cm]{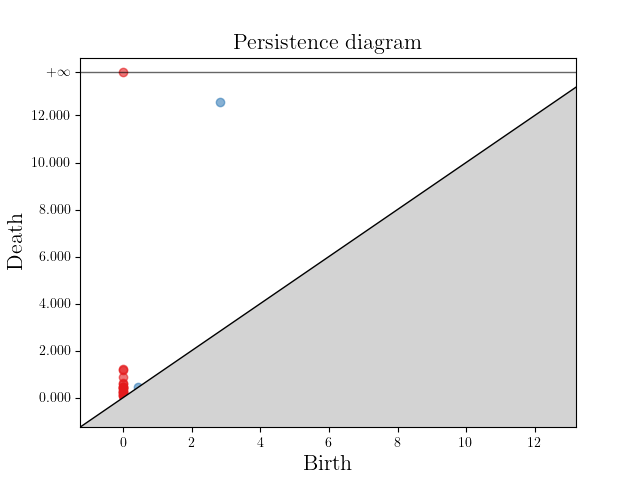}
    \caption{Persistence diagram}
        \label{fig:PD_cech}
    \end{subfigure}
	\caption{Balls with varying radius $t>0$ centered at each point of a finite subset $\X\subseteq \R^2$. These balls are used to define the \Cech{} filtration $\CechCplx(\X)$ and its corresponding persistence diagrams of dimension 0 (in red) and 1 (in blue).}
	\label{fig:Cech-filtration}
\end{figure}
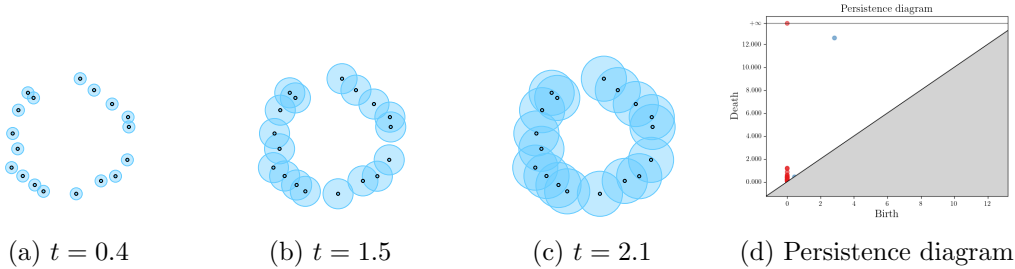

\subsection{Euler characteristic tools}\label{sec:euler_tools}
In this section, we recall the definitions of the descriptors of filtered simplicial complexes we use to perform topological data analysis. These invariants are defined using Euler characteristic profiles \citep{beltramo2022euler,DG22} and topological and hybrid transforms of constructible functions \citep{S95,GR11,L22}. While these tools can be defined in the more general setting of o-minimal geometry, we focus on filtered simplicial complexes.

Given an $m$-parameter filtration, computing the Euler characteristic for every value of the parameter $t \in \mathbb{R}^m$ gives an integer-valued function on $\mathbb{R}^m$ that is a multi-scale descriptor of the evolution of the filtration with respect to $t$.
\begin{definition}
    The \emph{Euler characteristic profile} of an $m$-parameter filtration~$\filt$ is the map:
    \begin{equation*}
        \ECP: t\in\R^m \mapsto \chi(\filt_t).
    \end{equation*}
    The map $\ECP$ is usually refered to as the \emph{Euler characteristic curve} (ECC) of $\filt$ when~$m=1$ and as the \emph{Euler characteristic surface} (ECS) of $\filt$ when $m=2$; see~\citet{beltramo2022euler,DG22}.
\end{definition}

Figure \ref{fig:ex-filt-ecs} shows an Euler characteristic surface computed on an elementary example. Widely used in data analysis \citep{smith2021euler, DG22, beltramo2022euler, jiang2020weighted}, this simple descriptor has proven to be efficient in capturing meaningful information on the data at hand. However, as illustrated in the following sections, we are interested in more robust descriptors built from integral transformations.

\begin{figure}[!h]
    \centering  
    \begin{subfigure}[t]{0.45\linewidth}
        \centering
        \includegraphics[height=4cm]{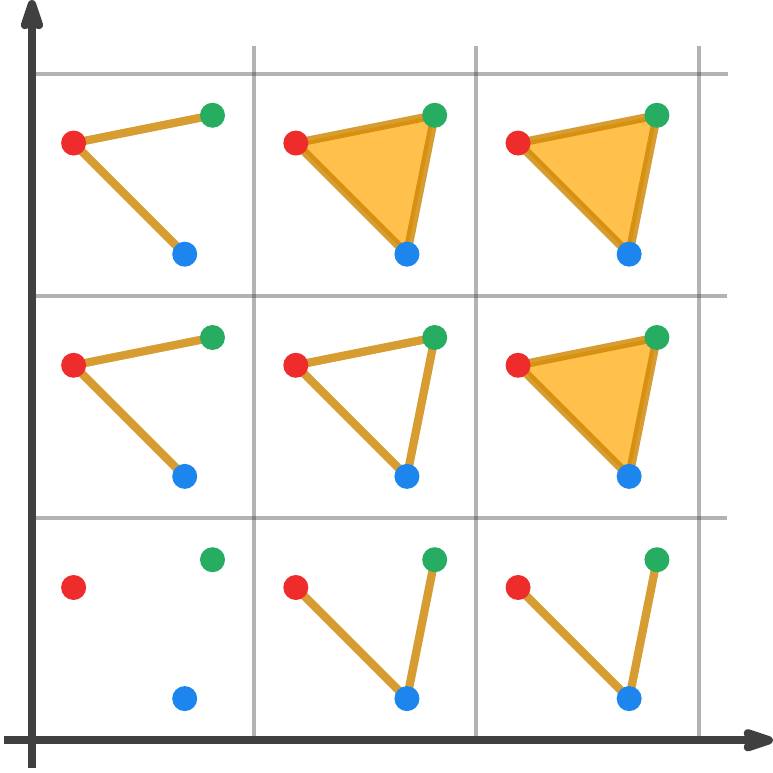}
        \caption{$\filt$}
        \label{fig:finitely-gen-filt}
    \end{subfigure}
    \begin{subfigure}[t]{0.45\linewidth}
        \centering
        \includegraphics[height=4.2cm]{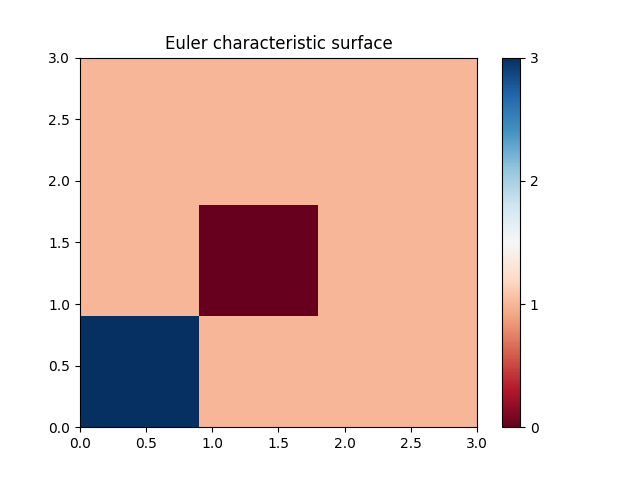}
        \caption{$\ECP$}
        \label{fig:toy-ECS}
    \end{subfigure}
    \caption{A finitely generated $2$-parameter filtration (a) and its associated Euler characteristic surface (b). All vertices have one birth time, while all other simplices have two.}
    \label{fig:ex-filt-ecs}
\end{figure}

Before introducing the other descriptors considered, we define the pushforward operation from Euler calculus; see \citet{S88, V88}: 
\begin{definition}\label{def:pushforward}
Let $\filt$ be a one-critical $m$-parameter filtration and $\xi\in\dual{\R_+^m}$. The \emph{pushforward of~$\filt$ along~$\xi$} is the one-parameter family defined for any $s\in\R$ by:
\begin{equation*}
    \big(\xi_*\filt\big)_s = \bigcup_{\dualdot{\xi}{t}\leq s}\filt_t.
\end{equation*}
The \emph{pushforward of $\ECP$ along $\xi$} is the Euler characteristic curve of $\xi_*\filt$. We denote this curve by $\xi_*\ECP$. 
In other words, we have $\xi_*\ECP = \ECP[\xi_*\filt]$. 
\end{definition}
One-criticality of the filtration is necessary for Definition \ref{def:pushforward} to match the classical definitions of \citet{S88, V88}. Writing the one-critical filtration as a sublevel-sets filtration, it is a simple exercise to check that the pushforward operation has a simple expression:
\begin{lemma}
\label{lem:dot_product}
    Let $f:\cplx\to\R^m$ be a non-decreasing map and $\xi\in\dual{\R_+^m}$. The Euler characteristic profile of~$\filt_f$ is denoted by~$\ECP[f]$. It is an easy exercise to check that~$\xi_*\filt_f = \filt_{\xi\circ f}$ and~$\xi_*\ECP[f] = \ECP[\xi\circ f]$.
\end{lemma}

Hybrid transforms mixing Euler calculus and classical Lebesgue integration have been introduced in \citet{L22}. These transforms are continuous and piecewise smooth and enjoy several beneficial properties, such as index theoretic formulae in the context of sublevel-sets persistence; see Propositions 4.1 and 4.2 and Theorem 8.3 in loc. cit.. In the present context, they can be defined as follows:
\begin{definition}
    Let $\filt$ be a one-critical $m$-parameter filtration and $\kernel \in \Lp{\R}$. The \emph{hybrid transform with kernel~$\kappa$} of~$\ECP$ is the map:
	\begin{equation*}
		\HT : \xi\in\dual{\R_+^m} \mapsto \int_\R \kernel(s)\xi_*\ECP(s)\ds.
	\end{equation*}
\end{definition}

The following lemma is an obvious consequence of Lemma \ref{lem:dot_product}. It states that any $m$-parameter hybrid transform restricted to an open half-line can be expressed as a one-parameter hybrid transform. It will be key to the proof of a law of large numbers for $m$-parameter hybrid transforms (\Cref{thm:LLN_multiD}).
\begin{lemma}
\label[lemma]{lem:HTn_to_HT1}
    Let $\filt$ be a one-critical $m$-parameter filtration, let $\kernel \in \Lp{\R}$ and $\xi\in\dual{\R_+^m}$. For any $\lambda>0$, one has:
    \begin{equation*}
        \HT[\filt](\lambda\xi) = \HT[\xi_*\filt](\lambda).
    \end{equation*}
\end{lemma}


Euler characteristic profiles and hybrid transforms constitute the two data descriptors we will use to perform topological data analysis. We give explicit expressions of these descriptors in specific cases below. These formulae will allow us to design algorithms to compute them in \Cref{sec:algorithm} and to build intuition on the type of behaviour they capture all along the paper.

\paragraph{One-critical filtrations.}
Up to reducing $\cplx$, one can assume that for any~$\sigma\in\cplx$, there is $t\in\R^m$ with $\sigma\in\filt_t$. Then, one has:
\begin{equation}\label{eq:ecp-one-critical}
    \ECP = \sum_{\sigma \in \cplx} (-1)^{\dim\sigma} \1_{Q_{t(\sigma)}},
\end{equation}
where $Q_u := \{t\in\R^m \st t \geq u\}$ for any $u\in\R^m$. 

Let $\kernel\in\Lp{\R}$. Denote by $\Kernel$ the primitive of $\kernel$ whose limit at $+\infty$ is 0. The hybrid transform with kernel $\kappa$ of $\ECP$ is:
\begin{equation}\label{eq:ht-one-critical}
    \HT:\xi\in\dual{\R_+^m} \longmapsto - \sum_{\sigma \in \cplx} (-1)^{\dim\sigma} \Kernel[\dualdot{\xi}{t(\sigma)}].
\end{equation}
\begin{remark}
    We often define hybrid transforms by specifying the primitive $\Kernel$ of the kernel $\kernel$ whose limit at $+\infty$ is 0. We call $\Kernel$ the \emph{primitive kernel} of the hybrid transform.
\end{remark}

Finally, in the case of a one-parameter filtration, hybrid transforms naturally appear as classical integral transforms of the Euler curve, making them a natural tool to extract information from the Euler curve and compress it into a small number of relevant coefficients. 

\paragraph{Connection with classical transforms.} 
Let $\filt$ be a one-critical $m$-parameter filtration. First, assume that $m=1$. For any $\xi\in\dual{\R}_+$ and any $s\in\R$, one has $(\xi_*\filt)_s = \filt_{s/\xi}$ and hence~$\xi_*\ECP(s) = \ECP(s/\xi)$. A change of variables then ensures that the hybrid transform with kernel $\kernel\in\Lp{\R}$ is equal to the rescaled classical transform:
\begin{equation}\label{eq:HT-as-classical}
    \HT: \xi\in\dual{\R}_+ \mapsto \xi \cdot \int_\R \kernel(\xi \cdot s)\ECP(s)\ds.
\end{equation}
Assume now that $m\geq 2$. The hybrid transform with kernel $\kernel$ differs from the classical integral transform:
\begin{equation*}
    \xi\in\dual{\R_+^m} \mapsto \int_{\R^m} \kernel(\dualdot{\xi}{x})\ECP(x)\d x.
\end{equation*}
We refer to \citet[Example~3.18]{L22} for a counter-example. In some special cases, however, such as when $\kernel(t)=\exp(-t)$, hybrid transforms and classical transforms coincide up to a rescaling \citep[Examples~5.12 and~5.17]{L22}. The interest in hybrid transforms over classical transforms can be motivated by the following example:

\begin{example}
    The one-parameter hybrid transform with kernel $\kernel(t) = \exp(-t)$ is also known as the \emph{persistent magnitude} \citep{GH21}. It is used in \cite{o2023alpha} as a new measure for estimating fractal dimensions of finite subsets $\X \subseteq \R^n$.
\end{example}

\subsection{Comparison of Euler characteristic tools with persistence diagrams}

Suppose that $\filt$ is a one-parameter filtration. In this case, Euler characteristic curves and hybrid transforms can simply be written as statistics of persistence diagrams. 
Denote the degree $k$ persistence diagram of $\filt$ by~$\diagram_k = \{(a_i, b_i)\}_{i=1}^{n_k}$ where $-\infty < a^k_i < b^k_i \leq \infty$ and an integer~$n_k\geq 0$. There exists $k_0$ such that persistence diagrams $\diagram_k$ are empty for all $k \geq k_0$. It is then straightforward to check that:
\begin{equation}
\label{eq:PD_ECP}
    \ECP = \sum_{k\geq 0} \sum_{i=1}^{n_k} (-1)^k \1_{[a^k_i,b^k_i)}.
\end{equation}
Let $\kernel\in\Lp{\R}$ and consider a primitive $\Kernel$ of $\kernel$. The hybrid transform with kernel~$\kappa$ of~$\ECP$ therefore writes as:
\begin{equation}
\label{eq:PD_HT}
    \HT:\xi\in\dual{\R}_+ \mapsto \sum_{k\geq 0} \sum_{i=1}^{n_k} (-1)^k \Big(\Kernel[\xi\cdot b^k_i] - \Kernel[\xi\cdot a^k_i]\!\Big),
\end{equation}
with the convention that $\Kernel(\xi\cdot b_i^k)$ is the limit of $\Kernel$ at $+\infty$ when $b_i^k = +\infty$. This connection between persistence diagrams and the one-parameter descriptors used in this article will be used for interpretation and in the asymptotic results of \Cref{sec:stat}. 

As we can see from \eqref{eq:PD_ECP} and \eqref{eq:PD_HT}, considering Euler curves and hybrid transforms instead of persistence diagrams implies a loss of information. More precisely:
\begin{itemize}
	\item Both Euler characteristic curves and hybrid transforms can be written as an alternated sum over all homological degrees. As a consequence, the information contained in persistence diagrams is summed up across all homological degrees and reduced to a single descriptor. 
	\item Even if only one persistence diagram is non-empty, the birth-death pairing of the points is lost while computing the Euler characteristic curve or hybrid transforms. In other words, Euler curves and hybrid transforms only depend on the sets $\{a_k^i\}$ and $\{b_k^i\}$ of all births and deaths respectively.
	\item Worse still, Euler curves are defined using indicator functions. As a consequence, the persistence diagrams $\{(0,1)\}$ and $\{(0,1/2),(1/2,1)\}$ share the same Euler curve and the same hybrid transforms for all kernels. Therefore, the lifetime of a feature $b-a$, usually used as an indicator of the significance of the point $(a,b)$ in the diagram is inaccessible.
\end{itemize}

The purpose of the following sections is to show that this loss of information does not result in a loss of accuracy when using Euler curves and hybrid transforms in machine learning tasks. Moreover, we show that the computation time is greatly reduced. This is due to the fact that Euler curves and hybrid transforms are computed using \eqref{eq:ecp-one-critical} and \eqref{eq:ht-one-critical}, bypassing the computation of homology and of persistence diagrams. This theoretical fact is backed up by experiments in \Cref{sec:timing}. Finally, we use the fact that Euler curves and hybrid transforms can naturally be adapted to filtrations with more than one parameter, while there is no analogues of \eqref{eq:PD_ECP} and \eqref{eq:PD_HT} in this case.
We believe that these major gains indicate that such descriptors should be preferred over persistence diagrams when tackling machine learning problems. 

\section{Method}\label{sec:method}
In this section, we describe the algorithms used to compute our descriptors and their implementation. We also give some intuition on choosing the kernel of hybrid transforms. Finally, we give heuristics on the type of information captured by Euler curves and their transforms on synthetic data sets.

\subsection{Algorithm}\label{sec:algorithm}
In every experiment, we directly compute Euler characteristic profiles and their hybrid transforms using formulae \eqref{eq:ecp-one-critical} and \eqref{eq:ht-one-critical}. Each algorithm takes as input a grid of size $d_1\times\ldots \times d_m$ on which the Euler characteristic profile or the hybrid transform is evaluated. The output array of size~$d_1\times\ldots\times d_m$ is an exact sampling of the descriptor. Therefore, our topological descriptors vectorize $m$-parameter filtrations into $d_1\times\ldots\times d_m$ arrays that can be used as input to any classical machine learning algorithm.

\paragraph{Complexity.}
The algorithm computing Euler characteristic profiles with resolution $d_1\times\ldots \times d_m$ has time complexity $\mathcal{O}(\card{K} + d_1\cdot\ldots\cdot d_m)$ in the worst case. The algorithm computing hybrid transforms with the same resolution has a worst-case time complexity of $\mathcal{O}(\card{K} \cdot d_1\cdot\ldots\cdot d_m)$. In comparison, computing a persistence diagram has time complexity $\mathcal{O}(|K|^\omega)$ in the worst case where $2 \leq \omega < 2.373$ is the exponent for matrix multiplication; see \citet{milosavljevic2011zigzag}.

\paragraph{Implementation.}
A \python{} implementation of our algorithms is freely available online on our GitHub repository: \url{https://github.com/vadimlebovici/eulearning}.
In practice, our implementation allows for two different ways of choosing a sampling grid. The first method takes as input bounds $[(a_1,b_1), \ldots, (a_m,b_m)]$ and a resolution~$d_1\times\ldots\times d_m$. We then compute a sampling of our descriptors on a uniform discretization of the subset $[a_1,b_1]\times\ldots\times[a_m,b_m]\subseteq \R^m$. This method has the disadvantage of requiring prior knowledge about the data. For Euler characteristic profiles, the second method consists in giving as input a list $[(p_1,q_1), \ldots, (p_m, q_m)]$ with real numbers $0\leq p_i < q_i \leq 1$. The algorithm then computes the $p_i$-th and the $q_i$-th percentiles of the $i$-th filtration for each $i=1,\ldots,m$. Finally, the Euler profiles are uniformly sampled on a $d_1\times\ldots\times d_m$ grid ranging from the lowest to the highest percentile on each axis. For the hybrid transforms, we provide a list~$[p_1, \ldots, p_m]$ of real numbers $0\leq p_i \leq 1$ and a positive real number $\alpha$. The algorithm then computes the $p_i$-th percentiles $v_i$ of the $i$-th filtration for each $i=1,\ldots,m$. The integral transforms are uniformly sampled on a $d_1\times\ldots\times d_m$ grid ranging from $0$ to $\alpha/v_i$ on each axis. This method does not require any prior knowledge of the data but depends on a choice of parameters. We give a few heuristics regarding the choice of the kernel function below.

\paragraph{Kernel choice.}
To interpret integral transforms of Euler curves, we set~$m=1$ and compute them on the simple function $\ECP = \1_{[a,b)}$ with~$a<b\in(0,+\infty)$. Recall that the hybrid transform has the simple expression~\eqref{eq:PD_HT}. Figure \ref{fig:kernel_ex} shows the hybrid transforms for several kernels. For every $p > 0$, the hybrid transform with primitive kernel $\Kernel: s \mapsto -\exp(-s^p)$ has a minimum in $\sqrt[p]{\frac{p(\log(b)-\log(a))}{b^p-a^p}}$, which tends to $1/b$ as~$p \to \infty$. As a consequence, transforms of this type yield \emph{smoothed} versions of the curve~$t\mapsto \ECP(1/t)$, that is, of an Euler curve with \emph{inverted scales}. Similarly, the hybrid transform with primitive kernel $\Kernel: s \mapsto -s^p \exp(-s^p)$ has a minimum that tends to~$1/a$ and a maximum that tends to~$1/b$ as~$p \to \infty$, with a spikier aspect as~$p \to \infty$. Transforms of this type record the \emph{variations} of the Euler characteristic curve with inverted scales. We refer to the following section for more involved experiments on synthetic data.

\begin{figure}
    \centering
    \begin{subfigure}[t]{0.31\linewidth}
        \centering
        \includegraphics[scale=0.33]{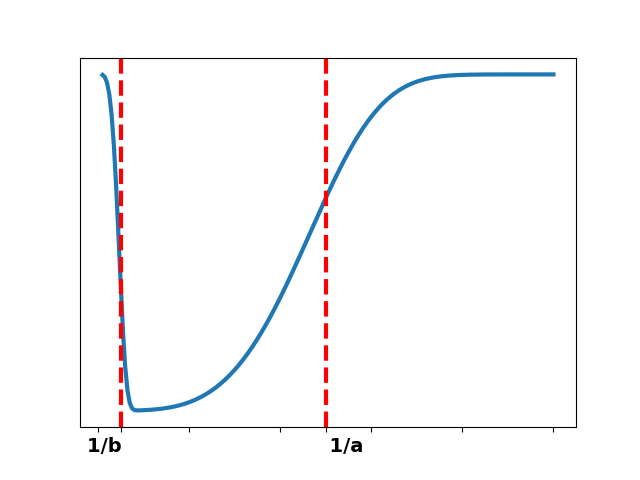}
        \caption{$\Kernel(s) = \exp(-s^4)$}
    \end{subfigure}
    \begin{subfigure}[t]{0.31\linewidth}
        \centering
        \includegraphics[scale=0.33]{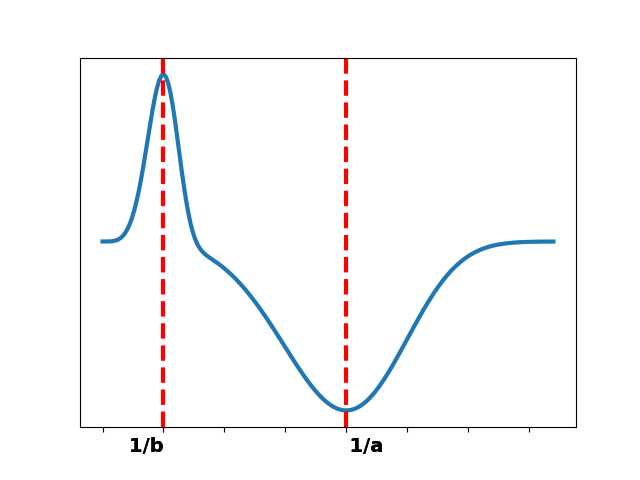}
        \caption{$\Kernel(s) = s^4\cdot\exp(-s^4)$}
    \end{subfigure}
    \begin{subfigure}[t]{0.31\linewidth}
        \centering
        \includegraphics[scale=0.33]{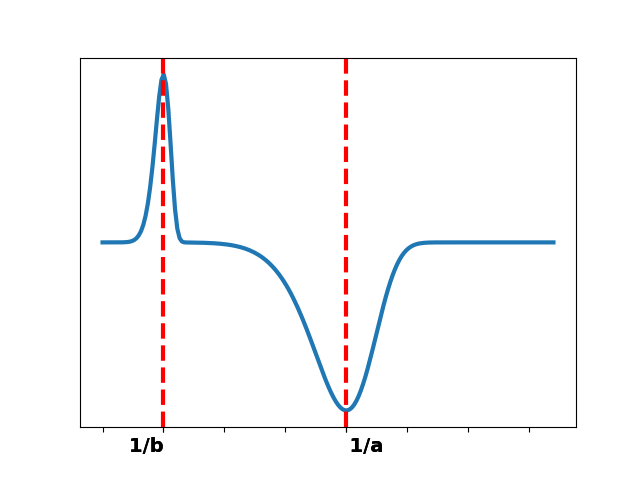}
        \caption{$\Kernel(s) = s^8\cdot\exp(-s^8)$}
    \end{subfigure}
    \caption{Hybrid transforms of $\ECP = \1_{[a,b)}$ for several choices of kernel $\kernel$}
    \label{fig:kernel_ex}
\end{figure}

\subsection{Heuristics for the Euler curves and their transforms}
\label{sec:heuristics}
In this section, we assume that $m=1$ and study the Euler characteristic curves associated with the filtered \Cech{} complex of a point cloud and the hybrid transforms of these curves. We overview how these descriptors can extract information about the input data's topology, geometry, and sampling density. As already mentioned in \Cref{ex:Cech}, we instead use alpha filtration in numerical experiments for computational reasons.

\subsubsection*{Poisson and Ginibre point processes}

While apparently coarse descriptors as opposed to persistence diagrams, Euler characteristic curves allow us to extract relevant scales at which topological differences between two different processes are revealed. To illustrate this claim, we try to discriminate between two types of point processes: a Poisson point process (PPP) and a Ginibre point process (GPP). This setup has been introduced in \citet{obayashi2018persistence}. Using topological descriptors to classify point processes has been extensively studied in \cite{biscio2020testing}. Ginibre processes imply repulsive interactions between points. While a standard PPP could have some very small and very large cycles, we expect the GPP to have more medium-sized cycles since points tend to be well dispersed. Ginibre point processes are generated using~\citet{decreusefond2021optimal}. We classify this toy data set with a random forest classifier and select the two scales corresponding to maxima of the \textit{feature importance} function of the classifier. In \Cref{fig:ex_PP}, we plot two examples of point clouds together with their alpha complexes at these scales.
\begin{figure}[H]
    \centering
    \begin{subfigure}[t]{0.24\linewidth}
        \centering
    	\includegraphics[scale=0.28]{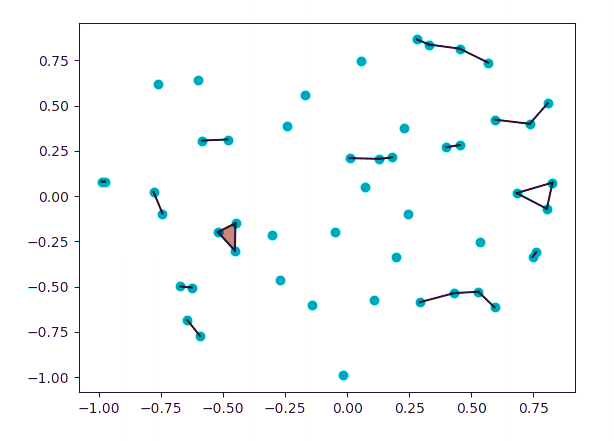}
        \caption{PPP}
    \end{subfigure}
    \begin{subfigure}[t]{0.24\linewidth}
    \centering
    	\includegraphics[scale=0.28]{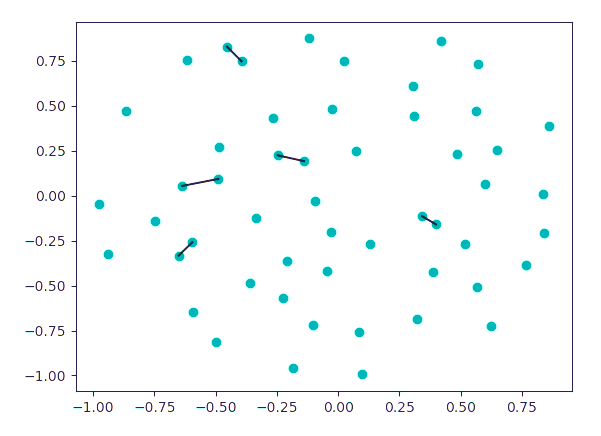}
    	\caption{GPP}
    \end{subfigure}
    \begin{subfigure}[t]{0.24\linewidth}
    \centering
    	\includegraphics[scale=0.28]{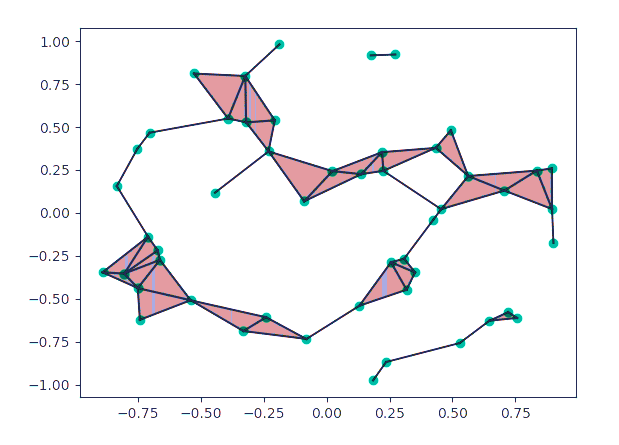}
    	\caption{PPP}
    \end{subfigure}
    \begin{subfigure}[t]{0.24\linewidth}
    \centering
    	\includegraphics[scale=0.28]{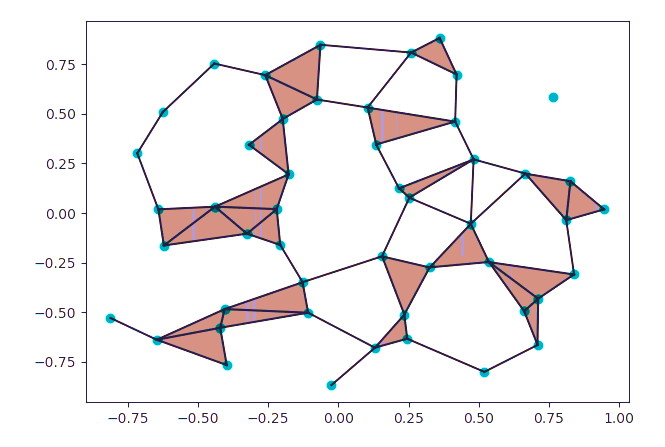}
    	\caption{GPP}
    \end{subfigure}
    \caption{Examples of alpha complexes on PPP and GPP point clouds at two scales~$t_1$ (Figures (a) and (b)) and~$t_2$ (Figures (c) and (d)) with~$t_1 < t_2$.}
    \label{fig:ex_PP}
\end{figure}

We plot Euler curves in \Cref{fig:GPP-PPP-ECC}. The Euler curves suggest that these classes differ at different scales, as it was visible in \Cref{fig:ex_PP}:
\begin{itemize}
    \item The Euler curves of the PPP class decrease more steeply. Indeed, a GPP has repulsive interactions between the points. Therefore, the pairwise distance between points tends to be larger and connected components do not die too early.
    \item The global minimum for the GPP class is lower.
    \item Compared to curves of the GPP class, the curves of the PPP class tend to stay negative for a longer time. Indeed, PPP allows for very large cycles to exist since there will typically be some large zones without any point, which is proscribed by GPP.
\end{itemize}
We remark that as opposed to persistence diagrams, our approach uses the birth times of edges instead of the usual degree $1$ homological features. It seems that this information suffices to discriminate between the two classes.

We plot the transforms of these curves for several kernels in \Cref{fig:GPP-PPP-HT-exp,fig:GPP-PPP-HT-exp_4}. Choosing the primitive kernel $\Kernel: s \mapsto \exp(-s)$ emphasises the small scales of the Euler curves in the larger scales of the transform. Such a descriptor separates well the two classes due to the earlier death of connected components for the PPP class. The primitive kernel $ \Kernel: s \mapsto \exp(-s^4)$ also extracts this information. In addition, it has a higher global maximum for the GPP class that also enables distinction between the two classes. This maximum is created by the global minimum of the Euler curves. This experiment is a piece of evidence that this kernel carries more information than the exponential kernel and will therefore be preferred for applications.

\begin{figure}[H]
\centering
\begin{subfigure}[t]{0.3\linewidth}
	\includegraphics[scale=0.37]{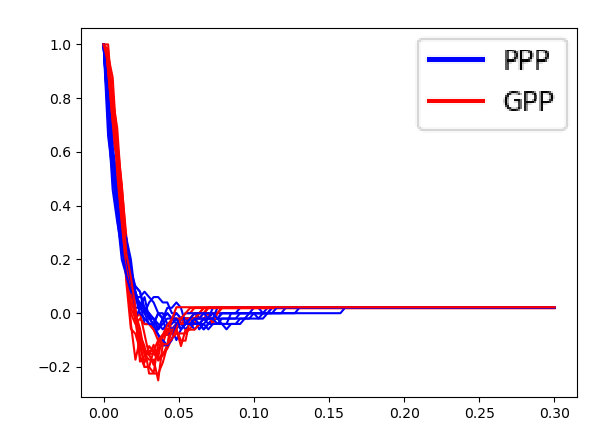}
	\caption{ECC}
	\label{fig:GPP-PPP-ECC}
\end{subfigure}
\begin{subfigure}[t]{0.3\linewidth}
	\includegraphics[scale=0.37]{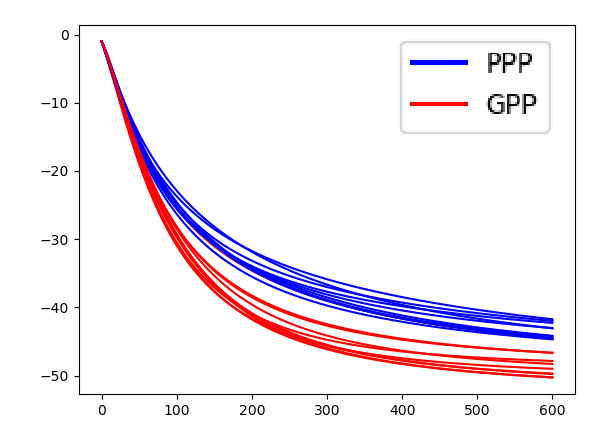}
	\caption{HT, $\Kernel(s) = \exp(-s)$}
	\label{fig:GPP-PPP-HT-exp}
\end{subfigure}
\begin{subfigure}[t]{0.3\linewidth}
	\includegraphics[scale=0.37]{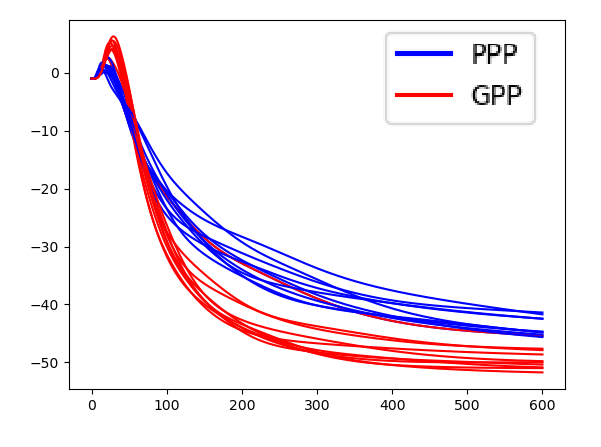}
	\caption{HT, $\Kernel(s) = \exp(-s^4)$}
	\label{fig:GPP-PPP-HT-exp_4}
\end{subfigure}
	\caption{Euler characteristic curves and their transforms for PPP VS GPP data set}
	\label{fig:PP_illus}
\end{figure}

\subsubsection*{Different samplings on a manifold}

We consider point clouds sampled on manifolds. We illustrate how our various descriptors can separate the information coming from the sampling density and the one coming from the shape of the support manifold. We consider two set-ups. The first one consists of clouds of 500 points sampled in two different ways on a torus embedded in $\mathbb{R}^3$. The first sampling is uniform \citep{diaconis2013sampling}. The second is a non-uniform sampling where we draw~$(\theta, \phi)$ uniformly in $[0, 2 \pi]^2$ and obtain a point on the torus through the embedding $\Psi_{\mathbb{T}^2}: (\theta, \phi) \mapsto (x_1, x_2, x_3)$ where:
\begin{equation*}
    \left\{
        \begin{array}{lll}
            x_1 = (2+\cos(\theta))\cos(\phi), \\
            x_2 = (2+\cos(\theta))\sin(\phi), \\
            x_3 = \sin (\theta).
        \end{array}    
    \right.
\end{equation*}
The second set-up consists of clouds of 500 points drawn in two ways on the unit sphere of $\R^3$. The first sampling is uniform. The second sampling is a non-uniform sampling where we draw $\theta$ uniformly on $[0,\pi]$ and $\phi$ according to a normal distribution centred on~$\pi$. We obtain a point on the sphere via the classical spherical coordinates parametrization~$\Psi_{\mathbb{S}^2}: (\theta, \phi) \mapsto (x_1, x_2, x_3)$ where:
\begin{equation*}
    \left\{
        \begin{array}{lll}
            x_1 = \sin(\theta) \cos(\phi), \\
            x_2 = \sin(\theta) \sin(\phi), \\
            x_3 = \cos (\theta).
        \end{array}    
    \right.
\end{equation*}
In Figures \ref{fig:ECC_torus} and \ref{fig:HT_torus}, we show the Euler curves and their hybrid transforms with primitive kernel $\Kernel: s \mapsto \cos(s)$ for these two classes of samplings on the torus. Up to a rescaling, this corresponds to a Fourier sine transform. In Figure \ref{fig:HT_sphere}, we show the hybrid transforms for the two classes of samplings on the sphere.

In both cases, Euler curves associated with data drawn on the same manifold all have the same profile, with a minimum value that tends to be lower for the uniform sampling. Similarly, the oscillations of the transforms are in phase and have the same amplitude. However, from one manifold to another, the phase and amplitude of the oscillations of the transforms differ significantly. This suggests that they are related to global quantities and are signatures of the support manifold. In contrast, the sampling scheme shows up in the vertical shifts of the oscillations of the transforms. This interpretation allows us to go beyond the classical signal/noise dichotomy of persistence diagrams. Although it makes no doubt that this sampling information can be retrieved from the points close to the diagonal in the diagram, it is still unclear how to untangle the information on the sampling density itself and its support directly on the diagrams. 
\begin{figure}[hbtp]
    \centering
    \begin{subfigure}[t]{0.3\linewidth}
        \includegraphics[scale=0.34]{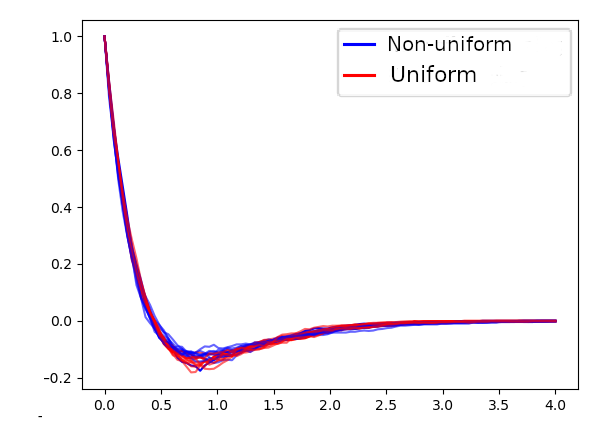}
        \caption{ECC, torus data}
        \label{fig:ECC_torus}
    \end{subfigure}
    \begin{subfigure}[t]{0.3\linewidth}
        \includegraphics[scale=0.34]{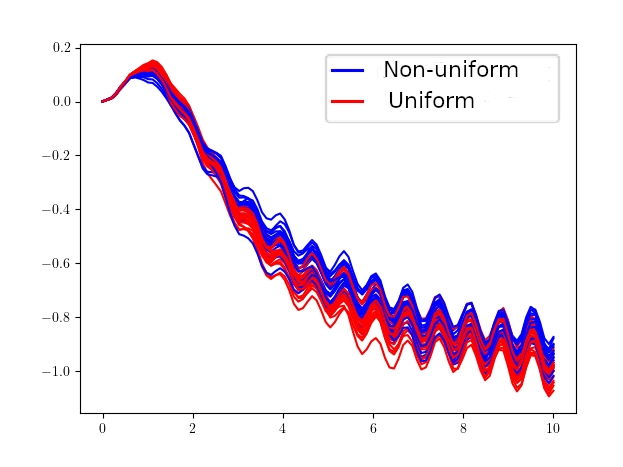}
        \caption{HT, torus data}
        \label{fig:HT_torus}
    \end{subfigure}
    \begin{subfigure}[t]{0.3\linewidth}
        \includegraphics[scale=0.34]{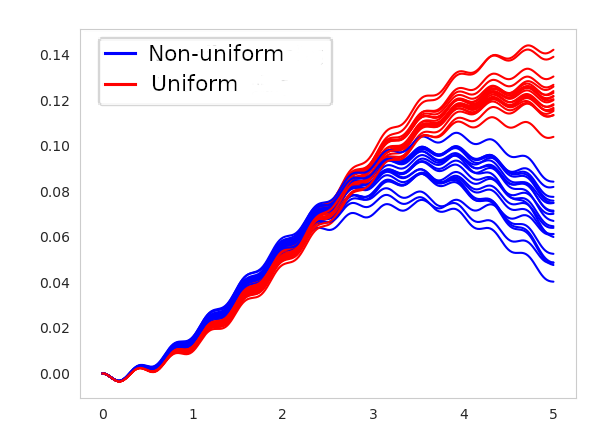}
        \caption{HT, sphere data}
        \label{fig:HT_sphere}
    \end{subfigure}
    \caption{ECC and HT, two sampling on the torus and the sphere}
    \label{fig:torus_illus}
\end{figure}

\subsubsection*{Signal in clutter noise}
In this final illustrative experiment, we try to distinguish patterns in a heavy clutter noise. One class has one line hidden in the noise, while the other has two. Each line will induce a very dense zone creating early dying connected components. In \Cref{fig:hidden_illus}, we plot two examples of point clouds, the Euler curves of each class, and their hybrid transform with primitive kernel $\Kernel: s \mapsto \exp(-s^4)$. We also provide PCA plots of these two descriptors. The difference between the two classes is visible at the beginning of the Euler characteristic curves. However, looking at the full curve does not allow us to correctly see this difference, as shown by the PCA plot. On the contrary, the transform puts a strong emphasis on the beginning of the Euler curves, leading to a direct linear separation of the two classes. As a final sanity check, we ran a k-means algorithm to cluster between the two classes and reached an accuracy of~$99 \%$ for the hybrid transforms and only $52.5 \%$ for the Euler curves.

\begin{figure}[hbtp]
	\centering
	\begin{subfigure}[t]{0.32\linewidth}
		\includegraphics[scale=0.35]{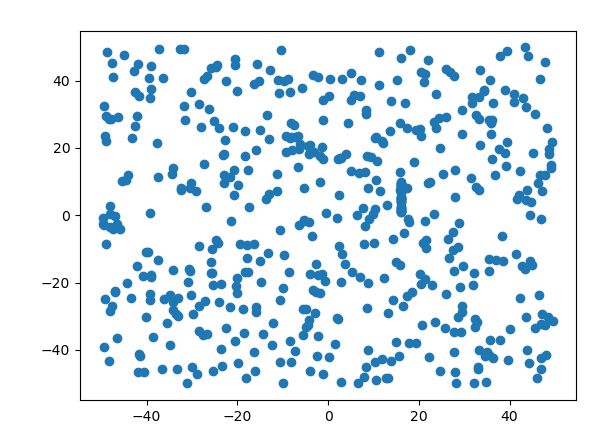}
		\caption{raw data, one hidden line}
	\end{subfigure}
	\begin{subfigure}[t]{0.32\linewidth}
		\includegraphics[scale=0.35]{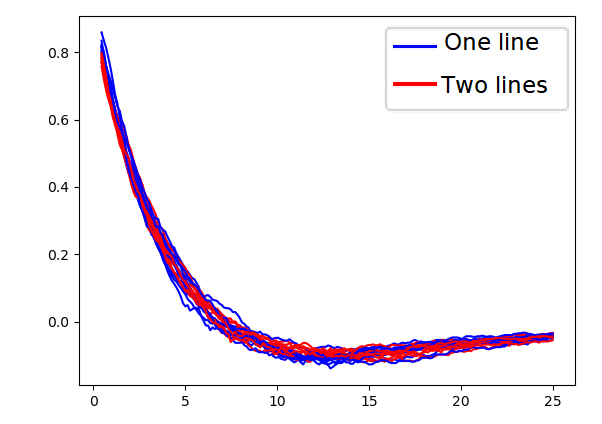}
		\caption{ECC}
	\end{subfigure}
	\begin{subfigure}[t]{0.32\linewidth}
		\includegraphics[scale=0.35]{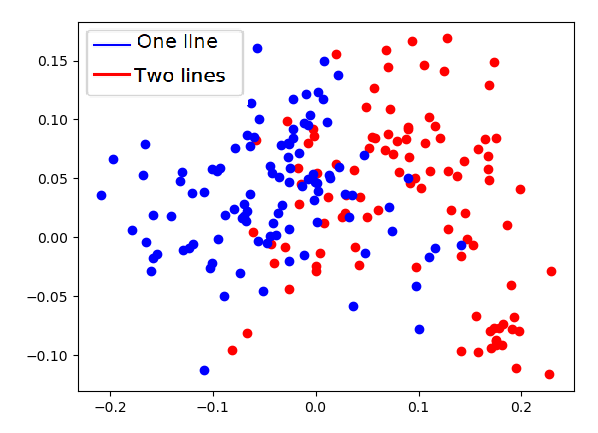}
		\caption{PCA on ECC}
	\end{subfigure}
	\\
	\begin{subfigure}[t]{0.32\linewidth}
		\includegraphics[scale=0.35]{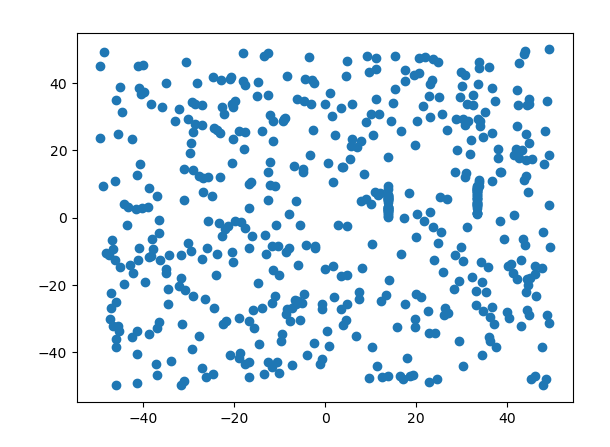}
		\caption{raw data, two hidden lines}
	\end{subfigure}
	\begin{subfigure}[t]{0.32\linewidth}
		\includegraphics[scale=0.35]{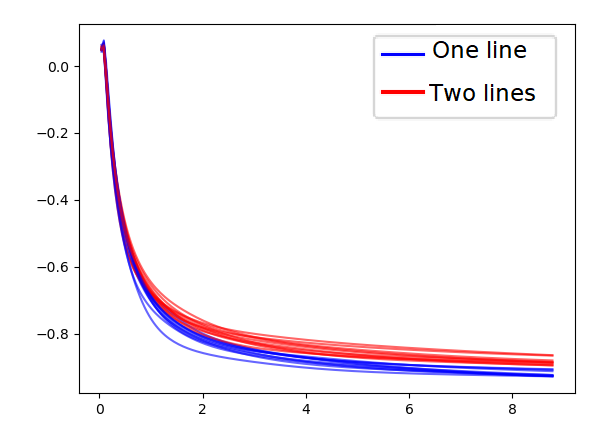}
		\caption{HT, $\Kernel(s)=\exp(-s^4)$}
	\end{subfigure}
	\begin{subfigure}[t]{0.32\linewidth}
		\includegraphics[scale=0.35]{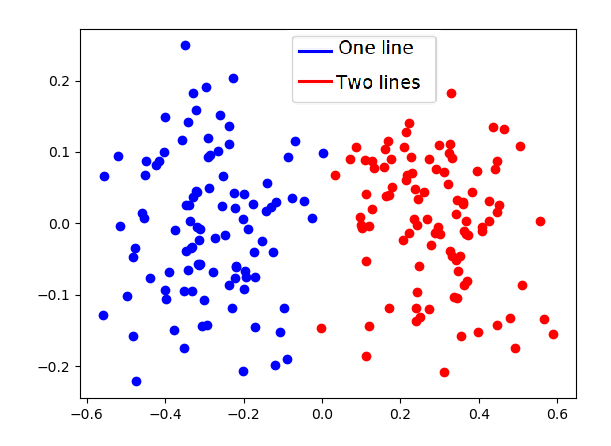}
		\caption{PCA on HT}
	\end{subfigure}
	\caption{Pattern hidden in clutter noise}
	\label{fig:hidden_illus}
\end{figure}

\section{Experiments}
\label{sec:expe}
In this section, we present all quantitative experiments conducted on synthetic and real-world point cloud data and on real graph data sets. Material to reproduce our experiments is available online on our GitHub repository: \url{https://github.com/vadimlebovici/eulearning}. Our timing experiments have been run on a workstation with an Intel(R) Core(TM) i7-4770 CPU (8 cores, 3.4GHz) and 8 GB RAM, running GNU/Linux (Ubuntu 20.04.1).

\subsection{Curvature regression}\label{sec:curv-reg}

We consider a set-up from \citet{bubenik2020persistent} where we draw $1000$ points uniformly at random on the unit disk of a surface of constant curvature $K$ and try to predict $K$ in a supervised fashion. Recall that if $K>0$ (resp. $K=0$, $K>0$), the corresponding surface is a sphere (resp. the Euclidean plane, the hyperbolic plane). We observe $101$ samples from the unit disk of the space with curvature~$[-2, -1.96, \ldots, 1.96, 2]$ and validate our model on a testing set of $100$ point clouds sampled from the disk of the space with random curvature drawn uniformly in $[-2, 2]$. We compare the $R^2$ scores in \Cref{score:curv} with that of the original paper, which uses persistent landscapes (PL) along with a support vector regressor (SVR) and with Persformer~\citep{reinauer2021persformer}. Note that since we are trying to tackle a regression problem, we use an SVR or a random forest regressor to predict the curvature from our vectorization.

\begin{table}[H]
\begin{center}
\begin{tabular}{ |c|c|c|c|c|c|c| } 

 \hline
 Method & PL+SVR & Persformer & ECC+SVR & ECC+RF & HT+SVR & HT+RF \\ 
 \hline
$R^2$ score & 0.78 & \textbf{0.94}  & 0.70 & 0.93 & 0.79 & 0.89 \\ 
\hline
\end{tabular}
\caption{$R^2$ score for curvature regression data}
\label{score:curv}
\end{center}
\end{table}

First, we remark that the ECC descriptor combined with a random forest has an accuracy comparable to state-of-the-art methods using persistence diagrams. We also remark that taking a transform does not improve the regression accuracy when considering a robust classifier such as RF but does improve the accuracy when using a linear regressor (SVR). Note that hybrid transforms combined with a linear regressor have an accuracy similar to that of persistent landscapes. However, persistent landscapes require the computation of the entire persistence diagrams, while hybrid transforms bypass this costly operation.

\subsection{\orbit{} data set}\label{sec:orbit}
\paragraph{Supervised classification.}

The \orbit{} data set is often used as a standard benchmark for classification methods in topological data analysis \citep{adams2017persistence,carriere2020perslay, reinauer2021persformer}. This data set consists of subsets of a thousand points in the unit cube $[0, 1]^2$ generated by a dynamical system that depends on a parameter $\rho>0$. To generate a point cloud, an initial point $(x_0, y_0)$ is drawn uniformly at random in $[0,1]^2$ and then the sequence of points $(x_n, y_n)$ for $n = 0, \ldots , 999$ is generated
recursively via the dynamic:
\begin{equation*}
\begin{array}{ll}
x_{n+1}=x_{n}+\rho y_{n}\left(1-y_{n}\right) & \bmod \hspace{0.5em} 1, \\[0.5em]
y_{n+1}=y_{n}+\rho x_{n+1}\left(1-x_{n+1}\right) & \bmod \hspace{0.5em} 1.
\end{array}
\end{equation*}%
In \Cref{fig:orbit_PC}, we illustrate typical orbits for $\rho \in\{2.5, 3.5,$ $4.0,4.1, 4.3 \}$. 

\begin{figure}[H]
    \centering
    \begin{subfigure}[t]{0.19\linewidth}
    \centering
        \includegraphics[scale=0.18]{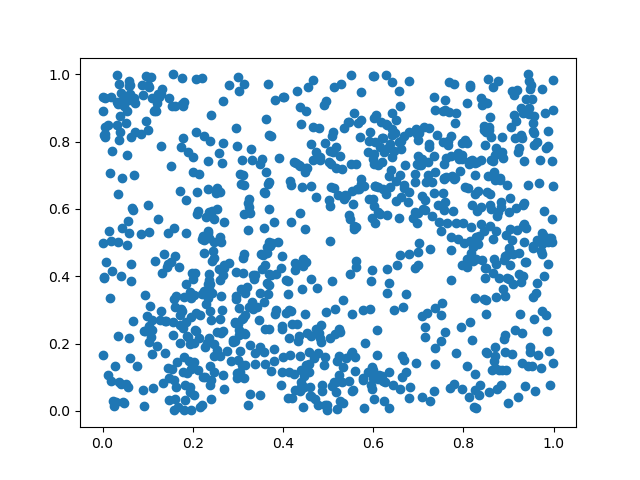}
        \caption{$\rho=2.5$}
    \end{subfigure}
    \begin{subfigure}[t]{0.19\linewidth}
    \centering
        \includegraphics[scale=0.18]{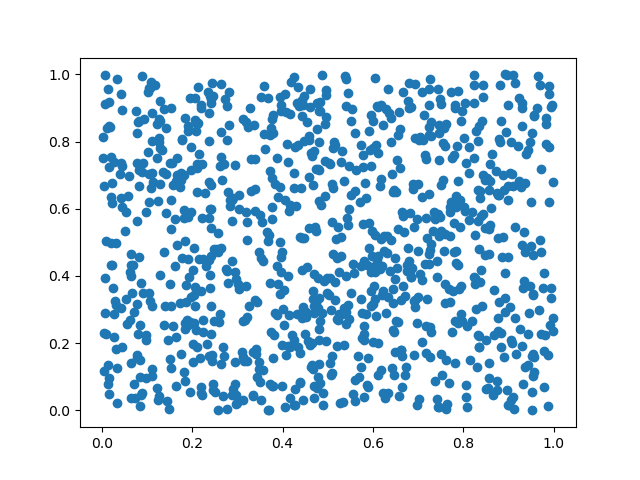}
        \caption{$\rho=3.5$}
    \end{subfigure}
    \begin{subfigure}[t]{0.19\linewidth}
    \centering
        \includegraphics[scale=0.18]{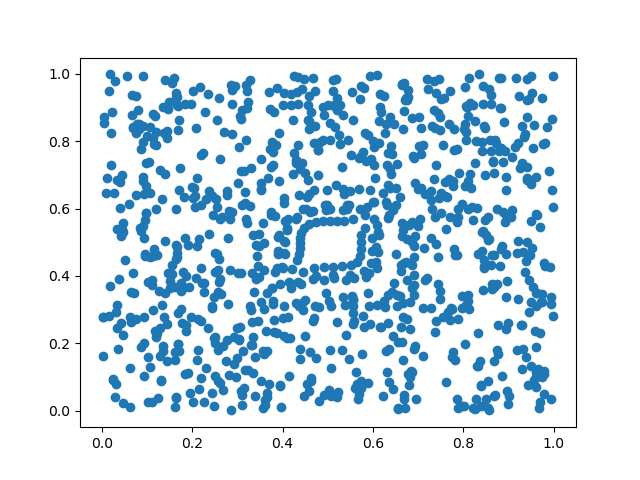}
        \caption{$\rho=4.0$}
    \end{subfigure}
    \begin{subfigure}[t]{0.19\linewidth}
    \centering
        \includegraphics[scale=0.18]{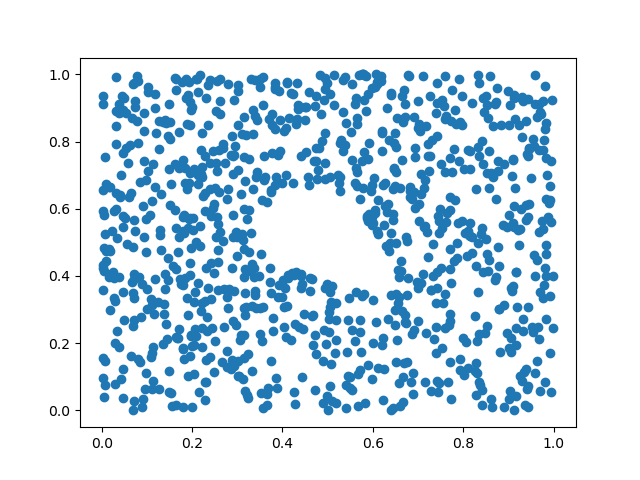}
        \caption{$\rho=4.1$}
    \end{subfigure}
    \begin{subfigure}[t]{0.19\linewidth}
    \centering
        \includegraphics[scale=0.18]{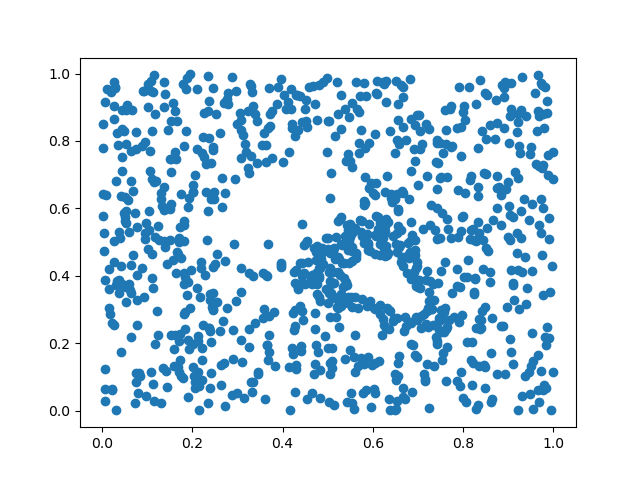}
        \caption{$\rho=4.3$}
    \end{subfigure}
    \caption{Examples of point clouds from the \orbit{} data set.}
    \label{fig:orbit_PC}
\end{figure}

Given an orbit of 1000 points, we try to predict the value of the parameter~$\rho$, which takes value in~$\{2.5, 3.5, 4.0, 4.1, 4.3 \}$. We generate 700 training and 300 testing orbits for each class. We compare our accuracy scores with standard classification methods using persistence diagrams in \Cref{tab:score-orbit}. The results are averaged over ten runs.  Sliced Wasserstein kernels (SW-K) and Persistence Fisher kernels (PF-K) are the two state-of-the-art kernel methods on persistence diagrams taken respectively from \citet{carriere2017sliced} and \citet{le2018persistence}. Perslay and Persformer are two methods that use a neural network architecture to vectorize persistence diagrams \citep{carriere2020perslay,reinauer2021persformer}. The Euler characteristic curves and one-parameter hybrid transforms (HT1) are computed on the alpha filtration of the point cloud. The Euler characteristic surfaces and two-dimensional hybrid transforms (HT2) are computed using a function-alpha filtration associated with a kernel density estimator post-composed with a decreasing function. The decreasing function is $x\mapsto -x$ for the ECSs and $x\mapsto \exp(-x^2)$ for the HTs. All descriptors have a resolution of 900 (hence of $30\times 30$ for two-parameter ones) and were classified using the XGBoost classifier \citep{chen2016xgboost}. We select the hyperparameters of our descriptors by cross-validation:
\begin{itemize}
    \item For the ECC, the quantiles (see \emph{Implementation} in \Cref{sec:algorithm}) are selected in $\{(0.1,0.9),$ $(0.2,0.8), (0.3,0.7)\}$.
    \item For the ECS, the quantiles are selected in the same set as for the ECC for both parameters.
    \item For the HT1, the range is selected in $\{[0,50],[0,100],[0,500], [0,1000]\}$ and the primitive kernel $\Kernel$ in $\{s\mapsto\exp(-s^4),$ $s\mapsto s^4\exp(-s^4), s\mapsto s^8\exp(-s^8)\}$.
    \item For the HT2, the primitive kernel and the range for the first parameter are the same as for the HT1, and the range for the second parameter is selected in $\{[0,50], [0,80],$ $[0,100],[0,500]\}$.
\end{itemize}

We show in \Cref{fig:2D_images} some examples of each descriptor renormalized by the number of points for the classes $\rho = 2.5$ and $\rho = 4.3$, where the HT2 is computed with $\Kernel: s \mapsto s^4 \exp(-s^4)$.

\begin{figure}[H]
	\centering
	\begin{subfigure}[t]{1\linewidth}
		\includegraphics[scale=0.36]{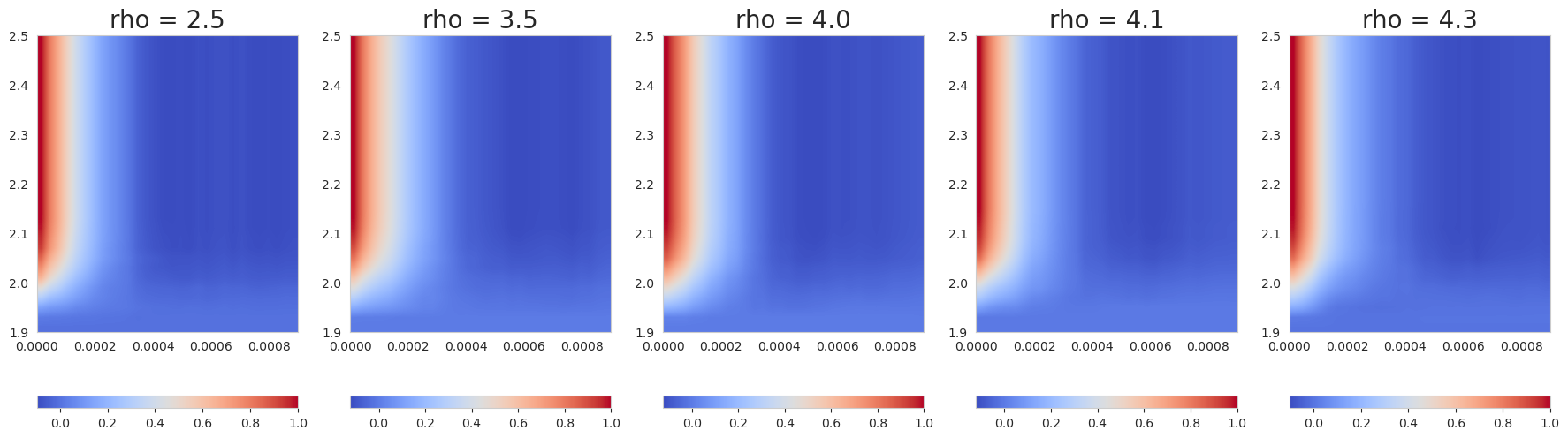}
		\caption{ECS}
	\end{subfigure}
	\\
    \begin{subfigure}[t]{1\linewidth}
		\includegraphics[scale=0.36]{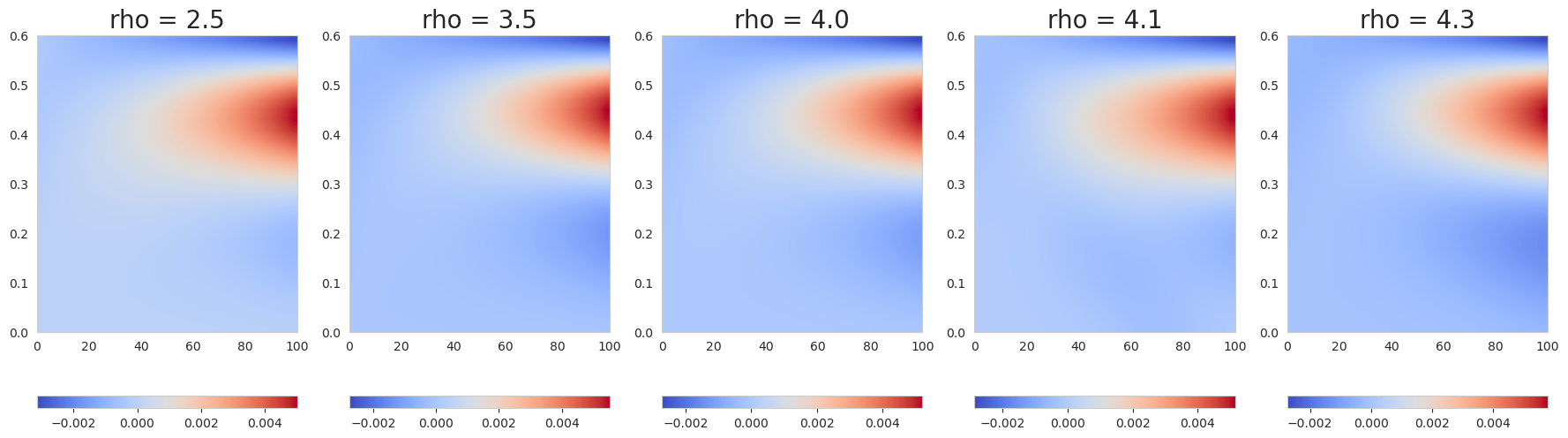}
		\caption{HT2}
	\end{subfigure}
	\caption{Examples of 2D descriptors}
	\label{fig:2D_images}
\end{figure}

\begin{table}[h]
\begin{center}
\hspace*{-0.3cm}
\begin{tabular}{ *{5}{|c}|} 
\hline
 Method 	& SW-K 				& PF-K 				& Perslay 			& Persformer \\ \hline
  Accuracy & 83.6 $\pm$ 0.9 	& 85.9 $\pm$ 0.8 	& 87.7 $\pm$ 1.0 	& 91.2 $\pm$ 0.8 \\ \hline\hline
 Method &  ECC + XGB 		& HT1 + XGB 		& \bestscore{ECS + XGB} 			& HT2 + XGB \\ \hline
  Accuracy & 83.8 $\pm$ 0.5 	& 82.8 $\pm$ 1.4 	& \bestscore{91.8 $\pm$ 0.4} 	 & 89.9 $\pm$ 0.5 	\\ \hline
\end{tabular}
\caption{Classification scores for the \orbit{} data set}
\label{tab:score-orbit}
\end{center}
\end{table}

One-parameter descriptors have accuracy similar to kernel methods on persistence diagrams at a reduced computational cost, while two-parameter descriptors compete with neural network-based vectorization methods. We make our claims on computational times more precise in \Cref{sec:timing}.

\paragraph{Ablation study.} We also study the role of the dimension of the feature vector in the supervised classification task. The results are shown in \Cref{fig:collapse-dim}. When plugging a random forest classifier, all descriptors are robust to a decrease in the size of the feature vector. However, hybrid transforms seem to maintain a competitive accuracy for low-dimensional features, especially the two-parameter ones. When using an SVM classifier for the one-parameter descriptors, the gain from considering a hybrid transform is clear, and the accuracy of the SVM benefits from this strong dimension reduction. Evaluating hybrid transforms at only three values of $\xi\in\dual{\R}_+$ yields feature vectors achieving approximately $80 \%$ accuracy, demonstrating the compression properties of this tool.

\begin{figure}[H]
	\centering
	\begin{subfigure}[t]{0.3\linewidth}
		\includegraphics[scale=0.35]{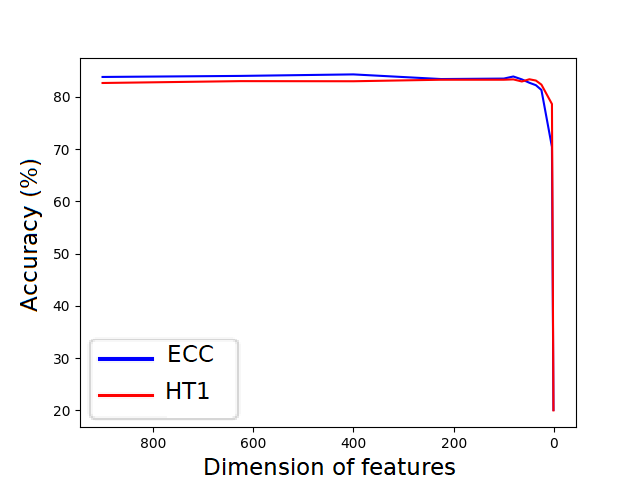}
		\caption{one-parameter, RF classifier}
	\end{subfigure}
	\begin{subfigure}[t]{0.3\linewidth}
		\includegraphics[width=0.91\linewidth]{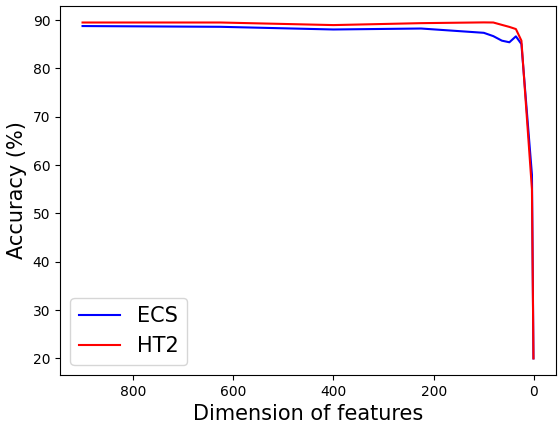}
		\caption{two-parameter, RF classifier}
	\end{subfigure}
	\begin{subfigure}[t]{0.3\linewidth}
		\includegraphics[scale=0.35]{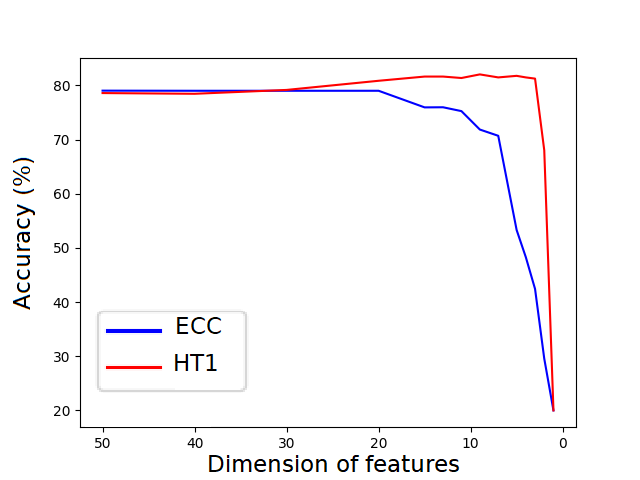}
		\caption{one-parameter, SVM classifier}
	\end{subfigure}
	\caption{Accuracy with respect to feature dimension.}
	\label{fig:collapse-dim}
\end{figure}

\subsection{Sydney object recognition data set}
\label{sec:sidney}
The Sydney urban objects recognition data set consists of 3D point clouds of everyday urban road objects scanned with a LIDAR \citep{de2013unsupervised} traditionally used for multi-class classification. Likewise to \Cref{sec:orbit}, all descriptors are computed using a function-alpha filtration associated with a kernel density estimator post-composed with a decreasing function.

\paragraph{Unsupervised setting.} In \Cref{fig:sid_pca}, we show a PCA of the ECSs and HTs on the classes \textit{4-wheeler vehicles} (labelled 0), \textit{buses} (2), \textit{cars} (3), and \textit{pedestrians} (4). In this case, the ECSs separate the class of pedestrians from all the vehicle classes. The same separation is achieved by the HTs with primitive kernel $\Kernel: s \mapsto s^4 \exp(-s^4)$. In contrast, HTs with primitive kernel $\Kernel: s  \mapsto \exp(-s^4)$ separate buses from other classes. These experiments illustrate the flexibility provided by a broad choice of kernels for the hybrid transforms. 

\paragraph{Supervised setting.}
Even more striking are the experiments from \Cref{fig:sid_lda}. We perform a Linear Discriminant Analysis for classes \textit{cars} (3), \textit{pedestrians} (4), and \textit{vans}~(13) to embed the HTs and ECSs in $\mathbb{R}^2$. All the classes are separated by the HTs with primitive kernel $\Kernel: s\mapsto s^4\exp(-s^4)$. In comparison, the ECSs only manage to separate the pedestrian class from the two motor-vehicle classes.

\begin{figure}[H]
	\centering
	\begin{subfigure}[t]{0.3\linewidth}
		\includegraphics[scale=0.36]{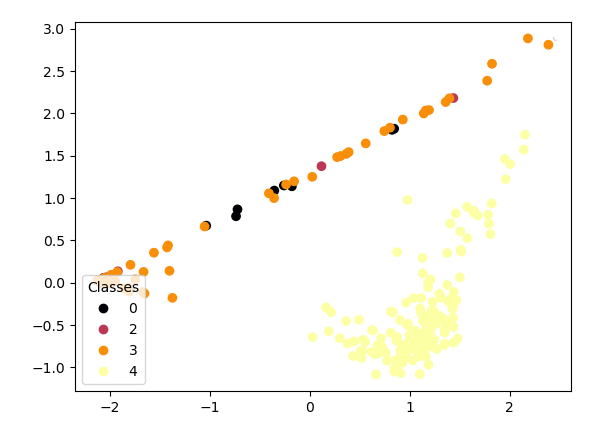}
		\caption{ECS}
	\end{subfigure}
	\begin{subfigure}[t]{0.3\linewidth}
		\includegraphics[scale=0.36]{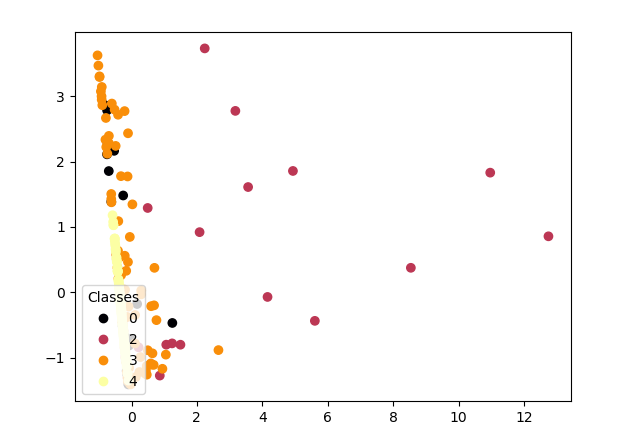}
		\caption{HTs, $\Kernel(s) = \exp(-s^4)$}
	\end{subfigure}
	\begin{subfigure}[t]{0.3\linewidth}
		\includegraphics[scale=0.36]{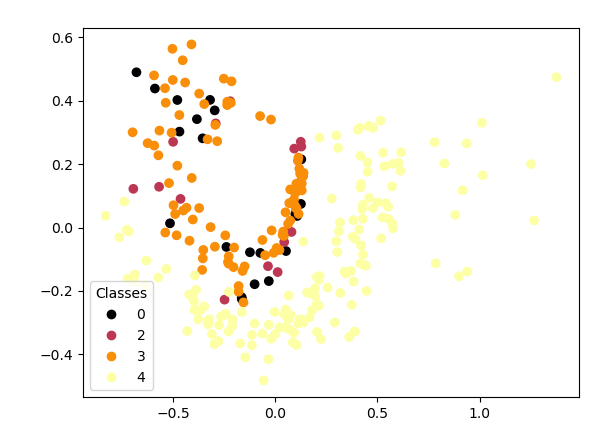}
		\caption{HTs, $\Kernel (s) = s^4 \exp(-s^4)$}
	\end{subfigure}
	\caption{PCA plots of ECSs and HTs for the Sydney object recognition data set.}
	\label{fig:sid_pca}
\end{figure}
	
\begin{figure}[H]
    \centering
    \begin{subfigure}[t]{0.32\linewidth}
        \centering
        \includegraphics[scale=0.35]{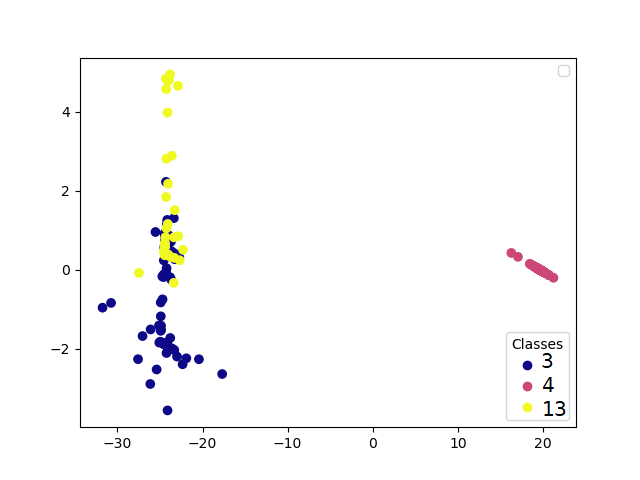}
        \caption{ECS}
    \end{subfigure}
    \begin{subfigure}[t]{0.32\linewidth}
            \centering
        \includegraphics[scale=0.35]{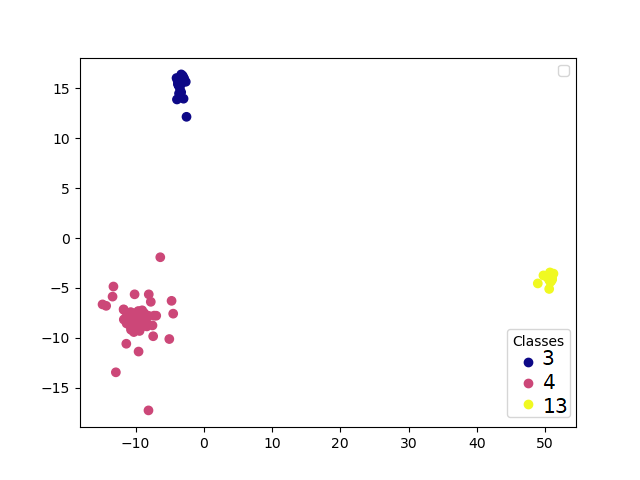}
        \caption{HTs, $\Kernel(s)=s^4\exp(-s^4)$}
    \end{subfigure}
    \caption{LDA plots of ECSs and HTs for the Sydney object recognition data set.}
    \label{fig:sid_lda}
\end{figure}

\subsection{Graph data}\label{sec:graph}
We have applied our method to the supervised classification of graph data. To build sublevel-sets filtrations of graphs, we consider the heat-kernel signature introduced in \citet{sun2009concise} and defined as follows. For a graph~$\mathcal{G} = (V, E)$, the \emph{HKS function with diffusion parameter $t$} is defined for each $v \in V$ by:
\begin{equation*}
\mathrm{hks}_t(v) = \sum_{k=1}^{|V|} \exp(-t \lambda_k) \psi_k(v)^2,
\end{equation*}
where $\lambda_k$ is the $k$-th eigenvalue of the normalized graph Laplacian and $\psi_k$ the corresponding eigenfunction. We consider the HKS with parameters $t=1$ and $t=10$ as filtrations. We also consider the $1/2-$Ricci and Forman curvatures \citep{samal2018comparative}, centrality, and edge betweenness on connected graphs. In addition, some data sets (\textsc{proteins}, \textsc{cox2}, \textsc{dhfr}) come with functions defined on the graph nodes. We can use several combinations of these functions to define sublevel-sets filtrations of graphs and compute Euler characteristic profiles (ECP) and hybrid transforms (HTn). 

For this set of experiments, we cross-validate over several combinations of the filtration functions proposed above, several truncations of the vectorization (which had little impact in practice), and a primitive kernel chosen among $\{ s \mapsto \cos(s), s \mapsto \cos(s^2), s \mapsto \exp(-s^4), s \mapsto s^4 \exp(-s^4)\}$ for HTn. We report our scores in Table \ref{score:graphs}. The first four methods are state-of-the-art classification methods on graphs that use kernels or neural networks. We report the scores from the original papers, \citet{tran2018scale, zhang2018retgk, verma2017hunt, xu2018powerful}. Perslay \citep{carriere2020perslay}, and Atol \citep{royer2021atol} are topological methods that transform the graphs into persistence diagrams using HKS functions. It is known that Atol performs especially well on large data sets (both in terms of number of data and graphs size), i.e., \textsc{collab} and \textsc{NCI1}. Still, we reach a similar to better accuracy for all the other data sets. 

Besides highly competitive classification scores, our method has two advantages over the other topological methods. First, we bypass the computation of persistence diagrams and thus classify with lower computational cost; see~\Cref{sec:algorithm,sec:timing}. Second, as opposed to other invariants such as multi-parameter persistent images~\citep{carriere2020multiparameter}, our method naturally generalizes to $m$-parameter persistence with $m\geq 3$ at a very low computational cost. To our knowledge, this is the first time a topology-based method uses more than 3 filtration parameters. This results in an increase in accuracy since each filtration function leverages information on the graph-data structures.

Note that the methods SV, FGSD, and GIN do not average ten times and rather consider a single 10-fold sample which can slightly boost their accuracies.

    \begin{table}[]
    \begin{center}
        \resizebox{\textwidth}{!}{\begin{tabular}{ |c||c|c|c|c|c|c|c|c|c|c|c| } 

\hline
Method & \textsc{mutag}  & \textsc{cox2} & \textsc{dhfr}  & \textsc{proteins}  & \textsc{collab} & \textsc{imdb-b}  & \textsc{imdb-m}  & \textsc{nci1}                                \\ \hline \hline
SV      & 88.2(0.1)   & 78.4(0.4)  & 78.8(0.7)  & 72.6(0.4) & 79.6(0.3)  & 74.2(0.9)                           & 49.9(0.3)                      & 71.3(0.4)                           \\ \hline
RetGK   & 90.3(1.1)                      & \textbf{81.4(0.6)} & 81.5(0.9)                           & \textbf{78.0(0.3)} & 81.0(0.3)                           & 71.9(1.0)                           & 47.7(0.3)                      & \textbf{84.5(0.2)} \\ \hline
FGSD    & \textbf{92.1} & -                                   & -                                   & 73.4      & 80.0                                & 73.6                                & \textbf{52.4} & 79.8                                \\ \hline
GIN     & 90(8.8)                        & -                                   & -                                   & 76.2(2.6) & 80.6(1.9)                           & \textbf{75.1(5.1)} & 52.3(2.8)                      & 82.7(1.6)                           \\ \hline
Perslay & 89.8(0.9)                      & 80.9(1.0)                           & 80.3(0.8)                           & 74.8(0.3) & 76.4(0.4)                           & 71.2(0.7)                           & 48.8(0.6)                      & 73.5(0.3)                           \\ \hline
Atol    & 88.3(0.8)                      & 79.4(0.7)                           & 82.7(0.7)                           & 71.4(0.6) & \textbf{88.3(0.2)} & 74.8(0.3)                           & 47.8(0.7)                      & 78.5(0.3)                           \\ \hline \hline
ECC 1D  & 87.2(0.7)                      & 78.1(0.2)                           & 79.4(0.5)                           & 74.7(0.4) & 77.3(0.2)                           & 72.4(0.4)                           & 48.5(0.3)                      & 74.4(0.2)                           \\ \hline
HT 1D   & 87.4(0.8)                      & 78.1(0.2)                           & 77.9(0.4)                           & 73.3(0.4) & 78.2(0.2)                           & 73.9(0.4)                           & 49.7(0.4)                      & 73.9(0.2)                           \\ \hline
ECP     & 90.0(0.8)                      & 80.3(0.4)                           & 82.0(0.4)                           & 75.0(0.3) & 78.3(0.1)                           & 73.3(0.4)                           & 48.7(0.4)                      & 76.3(0.1)                           \\ \hline
HT nD   & 89.4(0.7)                      & 80.6(0.4)                           & \textbf{83.1(0.5)} & 75.4(0.4) & 77.6(0.2)                           & 74.7(0.5)                           & 49.9(0.4)                      & 76.4 (0.2)   \\ \hline 

 
\end{tabular}}
        \captionof{table}{Mean accuracy and standard deviation for graph data.}
        \label{score:graphs}
        \end{center}
        
    \end{table}

\subsection{Timing}
\label{sec:timing}
 In this section, we compare the computational cost of our different methods to that of persistence images, a well-known vectorization of persistence diagrams introduced in \citet{adams2017persistence} and generalized to the multi-parameter setting in \citet{carriere2020multiparameter}. We choose to compare the computational cost of our methods to that of persistence images as they appear to be a faster vectorization method than persistence kernels and persistence landscapes; see \cite[Table~2]{carriere2020multiparameter}. 

\paragraph{Constant resolution.} We report in \Cref{tab:timing-cst-resolution} the time to compute our descriptors and persistent images on the full \orbit{} data set with a fixed resolution of 900. We assume that simplex trees are precomputed\footnote{Note that computing simplex trees takes around 66s in the one-parameter setting and around 420s in the two-parameter setting; the difference lies in the cost of computing codensity on point clouds.} using the \gudhi{} library \citep{gudhiAlphaComplex}. Our descriptors are computed using the parameters achieving the highest accuracy for the classification task; see \Cref{sec:orbit}. Persistence images are computed with the \gudhi{} library for one-parameter filtrations and with the \mma{} package for two-parameter filtrations \citep{MMA} with default parameters and the same resolution as our two-parameter descriptors, i.e., $30\times 30$. To compute persistence images, one first needs to compute the persistence diagrams of simplex trees in the one-parameter case or persistence approximations in the two-parameter case \citep[Section~3]{loiseaux2022efficient}. We include these additional costs in the computational times of persistent images. However, the time to compute the PI1 descriptor on the full \orbit{} data set breaks down to 5 seconds to compute the persistence diagrams and 134 seconds for the persistence images themselves.

\begin{table}[H]
    \centering
    \begin{tabular}{*{3}{|c}||*{3}{c|}} 
        \hline
        ECC & HT1  & PI1      & ECS     & HT2  & PI2 \\ \hline
        16  & 719   & 139       & 144    & 805   & 2034\\
        \hline
    \end{tabular}
    \caption{Computation times (s) for \orbit{} with constant resolution.}
    \label{tab:timing-cst-resolution}
\end{table}

As expected from the time complexities of the algorithms (\Cref{sec:algorithm}), Euler characteristic profiles are at least ten times faster than persistence images to compute, and hybrid transforms are four times faster in the two-parameter case. One-parameter hybrid transforms may appear costly to compute, but this point will be mitigated in the next paragraph. Finally, we point out that we implemented our tools in \python{} and not in \cpp{}, which is very likely to result in longer computation times. On the contrary, persistence images in one and two parameters both benefit from a \cpp{} implementation.

\paragraph{Constant accuracy.} We report in \Cref{tab:timing-cst-accuracy} the time to compute our descriptors on the full \orbit{} data set with the lowest resolution before accuracy drop-out as reported in \Cref{fig:collapse-dim}. More precisely, we chose the lowest possible resolutions to ensure a classification accuracy of 82\% for one-parameter descriptors and of 89\% for two-parameter descriptors, that is, a resolution of 30 for ECC, of 9 for HT1, of $20\times 20$ for ECS and of $6\times 6$ for HT2. Other parameters remain unchanged. The interest in using hybrid transforms over Euler characteristic profiles is now clear: the concentration of information provided by hybrid transforms makes it possible to classify the data set with feature vectors of reduced dimension, which considerably speeds up the computations. 

\begin{table}[!h]
\begin{center}
\begin{tabular}{*{2}{|c}||*{2}{c|}} 
\hline
ECC & HT1    & ECS     & HT2  \\ \hline
 16 &  5     & 135    & 69 \\
\hline
\end{tabular}
\caption{Computation times (s) for \orbit{} with smallest resolution before accuracy drop-out.}
\label{tab:timing-cst-accuracy}
\end{center}
\end{table}

\subsection{Take-home message}\label{sec:take-home}
The experiments from this section suggest that Euler characteristic profiles are very powerful descriptors since they allow for state-of-the-art accuracy when coupled with a robust classifier (XGB or RF) at a very competitive computational cost. Hybrid transforms have similar accuracy but are more costly to compute, especially in the one-parameter setting; see \Cref{tab:timing-cst-resolution}. The motivation to use hybrid transforms is two-fold:
\begin{itemize}
    \item In an unsupervised setting or when plugging a linear classifier, the lack of diversity in Euler characteristic profiles can be detrimental to the separation of classes. In contrast, hybrid transforms are competitive descriptors in such tasks due to the wide diversity in the choice of kernels and their sensitivity to slight variations in Euler characteristic profiles.
    \vspace{0.5em}
    \item Hybrid transforms provide a very powerful compression of the signal from the Euler profiles (\Cref{fig:collapse-dim}) at a meagre computational cost (\Cref{tab:timing-cst-accuracy}). This makes hybrid transforms robust descriptors combining dimension reduction and feature extraction.
\end{itemize}

Theoretically, multi-parameter hybrid transforms benefit from their expression as one-parameter ones (\Cref{lem:HTn_to_HT1}). This allows us to prove almost sure convergence results under some mild assumptions in \Cref{sec:stat}.

\subsection{Extensions}
We have validated our method on simplicial complexes built on point clouds and graph data. Nonetheless, the methodology described in this paper can be extended into two directions. 

First, when dealing with images or 3D volumes, it is common to build cubical complexes from data. In this context, Euler characteristic curves have been used as a vectorization of the data in \citet{smith2021euler, jiang2020weighted}. As there are a vast number of filtration functions one can consider on images, it is worth investigating the predictive power of the Euler characteristic profiles in this setting. While several applications are considered in \citet{richardson2014efficient, beltramo2022euler,DG22}, a thorough benchmark against other persistence methods and state-of-the-art image processing methods is still missing. Moreover, hybrid transforms have still not been studied in this context. 
    
Second, the methodology developed here applies to filtrations $\filt=(\filt_t)_{t\in\R^m}$ that are not necessarily non-decreasing with respect to inclusions. This extends the potential range of applications of our tools, notably to the study of time-varying simplicial complexes, as done in \citet{xian2022capturing}.

\section{Stability properties}\label{sec:stability}
The success of topological data analysis inherits from the stability theorem for persistence diagrams from \citet{cohen2007stability}. Loosely speaking, it means that under mild assumptions, small changes in the filtration function imply small changes in the diagram. Such results are crucial to designing consistent estimators; see, for instance, \citet{bobrowski2017topological}. Over the past decade, more distances on persistence diagrams have been introduced. Inspired by optimal transport theory, the notion of $p$-Wasserstein distance is introduced by~\citet{cohen2010lipschitz} where a stability result is also proven. A finer stability result for the $p$-Wasserstein distance can be found in \citet{skraba2020wasserstein}. In addition, several stability results for Euler characteristic tools have been derived in \citet{curry2022many,DG22,perez2022euler,meng24,marsh2023}.

In this section, we state stability results for our topological descriptors. Our results compare the $L^1$ norm between Euler characteristic profiles to the signed $1$-Wasserstein distance between their so-called \emph{signed barcodes}. As a corollary, we bound the $L^q$ norms of hybrid transforms by the same quantity. To continue our comparison with persistence diagrams, we prove that in the one-parameter case, the signed 1-Wasserstein distance between signed barcodes is bounded from above by the well-known 1-Wasserstein distance between persistence diagrams.

The notions of signed barcodes and of signed $1$-Wasserstein distance have been introduced in \citet{oudot2021stability} and are recalled below. We follow the same conventions as in \citet[Section~2]{oudot2021stability} for the definitions of multisets and bijections between them. The rest of the section is devoted to the statement of our stability results. All proofs are written in \Cref{sec:proofs-stability}.

\paragraph{Signed $1$-Wasserstein distance.}
The distance we use to state our stability results is defined on the class of \emph{finitely presented} functions over $\R^m$, that is, which can be written as a finite $\Z$-linear combination of indicator functions $\1_{Q_u}$ for some $u\in\R^m$. These functions include Euler characteristic profiles of finitely generated filtrations (\Cref{lem:ECP-is-FP}). We denote by $\FP$ the group of finitely presented functions over~$\R^m$. These functions have a kind of diagram (or barcode) that can be used to define an analogue of the $1$-Wasserstein distance. A \emph{decomposition} of $\phi\in\FP$ is a couple~$(\barcode^+, \barcode^-)$ of finite multisets of points in $\R^m$ such that:
\begin{equation*}
    \phi = \sum_{u\in\barcode^+} \1_{Q_{u}} - \sum_{v\in\barcode^-} \1_{Q_{v}}. 
\end{equation*}
Such a decomposition always exists, and there is a unique $\sbarcode = (\barcode^+, \barcode^-)$ such that~$\barcode^+\cap \barcode^- = \emptyset$, called the \emph{signed barcode of $\phi$}; see \cite[Proposition~13]{oudot2021stability}. While two different notions of signed barcode are defined in loc. cit., we focus here on the so-called \emph{minimal Hilbert decomposition signed barcode}. 

Let $\mathcal{C}$ and $\mathcal{C}'$ be two finite multisets of points in $\R^m$ with the same cardinality and~$h:\mathcal{C} \to \mathcal{C}'$ be a bijection between them. The \emph{cost} of $h$ is the real number $\cost(h) = \sum_{u\in\mathcal{C}} \|u - h(u)\|_1$. For any two finitely presented functions $\phi$ and $\phi'$ with respective signed barcodes $(\barcode^+, \barcode^-)$ and~$(\barcode'^+, \barcode'^-)$, the \emph{signed 1-Wasserstein distance} between them is:
\begin{equation*}
    \dist\big(\phi, \phi'\big) = \inf \left\{ \eps>0 \st \exists \textnormal{ bijection } h:\barcode^+\cup\barcode'^- \to \barcode^-\cup\barcode'^+ \textnormal{ with } \cost(h) \leq \eps \right\}.
\end{equation*}%
Hence, one has $\dist\big(\phi, \phi'\big)\in[0,+\infty]$. 
Note that bijections do not allow for unmatched bars, as it is common in the persistence literature. In loc. cit., the signed 1-Wasserstein distance is defined on signed barcodes. Our definition is essentially equivalent since signed barcodes are in one-to-one correspondence with finitely presented functions up to forgetting the order in the multisets.

\paragraph{Stability results.} 
We prove stability results involving functional norms on Euler characteristic profiles and their hybrid transforms. The case $m=1$ is well known for 1-Wasserstein distance on persistence diagrams; see \citet[Lemma~4.10]{curry2022many}, \citet[Proposition~3.2]{DG22}.
\begin{proposition}\label[proposition]{prop:stability-ECP}
    Let $\filt$ and $\filt'$ be two finitely generated $m$-parameter filtrations of simplicial complexes~$\cplx$ and $\cplx'$ respectively. For any $M>0$, we have that
    \begin{equation*}
        \|\ECP[\filt] - \ECP[\filt']\|_{1,M} \leq (2M)^{m-1} \, \dist[\ECP[\filt], \ECP[\filt']].
    \end{equation*}
    In particular, if $m=1$:
    \begin{equation*}
        \|\ECP[\filt] - \ECP[\filt']\|_{1} \leq \dist[\ECP[\filt], \ECP[\filt']].
    \end{equation*}
\end{proposition}
This stability result for Euler characteristic profiles implies a similar stability result for hybrid transforms, as stated in the following corollary.
\begin{corollary}\label[corollary]{cor:stability-Rdn-ht}
    Let $K$ be a compact subset of $\dual{\R_+^m}$ and $q\in[1,\infty]$. Let $\filt$ and $\filt'$ be one-critical $m$-parameter filtrations of simplicial complexes~$\cplx$ and $\cplx'$ respectively. Let~$\kernel\in\Lp{\R}\cap\Lp[\infty]{\R}$. There exists a constant $C_{K,q}$ depending only on $K$ and $q$ such that:
    \begin{equation*}
         \| \HT[\filt] - \HT[\filt']\|_{L^q_K} \leq C_{K,q} \, \|\kernel\|_\infty \, \dist[\ECP[\filt], \ECP[\filt']].
    \end{equation*}
\end{corollary}

Now, we prove two connections of the signed $1$-Wasserstein distance with more classical distances between filtrations. The first connection is made with the $1$-Wasserstein between persistence diagrams~\citep{cohen2010lipschitz}. We start by recalling it. Denote by~$\diagram$ and~$\diagram'$ the degree $k$ persistence diagrams of~$\filt$ and~$\filt'$. The \emph{$p$-Wasserstein distance} between $\diagram$ and $\diagram'$ is defined as:    
\begin{equation*}
W_p(\diagram, \diagram') = \underset{\eta}{\inf} \left(\sum_{x \in \diagram} \|x-\eta(x)\|^p \right)^{1/p}
\end{equation*}
where the infimum is taken over all bijections $\eta: \diagram \cup \Delta \to \diagram' \cup \Delta$ where $\Delta = \{(s, s) | s \in \R \}$ is the diagonal of $\R^2$. This definition allows for matchings between diagrams with different number of points.
%
We can now state the following connection between the $1$-Wasserstein distance on diagrams and the signed $1$-Wasserstein on Euler characteristic curves.
\begin{lemma}\label[lemma]{lem:dist-W1-barcode}
	Let $\filt$ and $\filt'$ be two finitely generated one-parameter filtrations of respective simplicial complexes~$\cplx$ and $\cplx'$. Denote their respective persistence diagrams $\diagram$ and $\diagram'$. Then,
	\begin{equation*}
		\dist[\ECP, \ECP[\filt']] \, \leq \ 2\, \sum_{k\geq 0} W_1\big(\diagram, \diagram'\big).
	\end{equation*}
\end{lemma}
Combined with \Cref{prop:stability-ECP} and \Cref{cor:stability-Rdn-ht}, this lemma ensures that $L^1$ norms of Euler characteristic curves and $L^q$ norms of their hybrid transforms are controlled by a classical distance between their persistence diagrams. This is another element of comparison between Euler characteristic tools and persistence diagrams. It is important to note that all homology degrees have to be taken into account for the result to hold.

The second connection is established between the signed $1$-Wasserstein distance on Euler characteristic profiles and $L^1$ norms on filtration functions defined on the same simplicial complex, as stated by the lemma below. It has already been formulated in a slightly different form in \citet[Proposition~3.4]{DG22}. Let $\cplx$ be a finite simplicial complex, and $f:\cplx \to \R^m$ a non-decreasing map. We define the \emph{$1$-norm} of $f$ as~$\|f\|_1 = \sum_{\sigma\in\cplx} \|f(\sigma)\|_1$.
\begin{lemma}\label[lemma]{lem:stability-L1-filt-fn}
    Let $\cplx$ be a finite simplicial complex and $f, g:\cplx \to \R^m$ be non-decreasing maps. We have that
    \begin{equation*}
        \dist[\ECP[f], \ECP[g]] \leq \|f-g\|_1.
    \end{equation*}
\end{lemma}

The above lemma clarifies the robustness of our descriptors with respect to perturbations of filtrations defined on a fixed simplicial complex. This includes, for instance, density estimators on point clouds or Ricci curvature and HKS functions on graphs. The fact that these descriptors are controlled by the $L^1$ distance and not the $L^\infty$ distance between functions is an indicator of their sensitivity to the underlying geometry. Persistent images \citep{adams2017persistence} share this property, while persistence landscapes \citep{bubenik2015statistical,vipondlandscapes} do not, as they are controlled by the $L^\infty$ distance between functions.

\section{Statistical properties}
\label{sec:stat}
This section provides statistical guarantees for our descriptors computed on a random sample $(X_1, \ldots, X_n)$ in $\R^d $, as the sample size $n$ tends to infinity. We adapt recent results on asymptotic persistence diagrams to our set-up, to prove some form of universality of our topological descriptors. For these asymptotic results to hold, the point cloud must be rescaled by a sequence $(r_n)_{n \in N}$ that tends to 0 as $n$ tends to infinity. The speed at which $r_n$ converges to $0$ determines the \textit{scaling regime} and can give different limit results. Theorem \ref{thm:LLN_HT1D} is established in the \textit{sparse} regime, i.e. such that $nr_n^d \to 0$ as $n \to \infty$, while Theorems \ref{TCL_ECC}, \ref{thm:TCL_HT} and \ref{thm:LLN_multiD} are established in the \textit{critical} (or thermodynamic) regime , i.e. $nr_n^d \to c \in ]0, \infty[$as $n \to \infty$. All the results are established for compactly supported kernels.

\subsection{Limit theorems for one-parameter hybrid transforms}
This section is devoted to limit theorems for the hybrid transforms of the \Cech{} complex of an i.i.d. sample in $\mathbb{R}^d$. \Cref{thm:LLN_HT1D} is a pointwise law of large numbers, while \Cref{thm:TCL_HT} establishes a functional central limit theorem for the hybrid transforms of compactly supported kernels. The purpose of this section is two-fold: we state that under some mild assumptions, hybrid transforms are universal in the sense that they converge to an object that depends only on the kernel, the filtration, and the sampling scheme. In addition, we demonstrate that as the sampling density appears explicitly in Theorems \ref{thm:LLN_HT1D} and \ref{thm:TCL_HT}, hybrid transforms can, at least asymptotically, be used to discriminate between samples from different probability densities. Theorem \ref{thm:LLN_HT1D} is a direct adaptation of Theorem 3.2 of \cite{owada2022convergence} that states an asymptotic result for persistence diagrams in the \textit{sparse regime}.

\begin{theorem}
\label{thm:LLN_HT1D}
Let $X_1, \ldots, X_n$ be an i.i.d. sample drawn according to an a.e. continuous bounded Lipschitz density $g$ on $\mathbb{R}^d$. Consider a sequence $(r_n)_{n\in\N}$ such that $n r_n^d \to 0$ and $n^{k+2}r_n^{d(k+1)} \to \infty$ as $n\to\infty$ for all $k$ in $\llbracket 0, d-1 \rrbracket$. We denote by $\filt_n$ the \Cech{} filtration associated with the rescaled sample $\frac{1}{r_n}(X_i)_{i=1}^n$. Let $T, a >0$ and $\kernel\in\Lp{\R}$. Further assume that $\kernel$ is supported on $[0, T]$. Then there exist functions $A_0, \ldots, A_{d-1}$ on $\dual{\R}_+$ that depend only on~$\Kernel$ such that for every $\xi >a$, 
\begin{equation*}
    \frac{1}{n^{k+2}r_n^{d(k+1)}}\cdot \HTn(\xi) \ \longtoinf\  \sum_{k=0}^{d-1} \frac{(-1)^k}{(k+2)!}\cdot A_k(\xi)\cdot\int_{\mathbb{R}^d} g^{k+2} (x) \d{}x \quad \text{ a.s..}
\end{equation*}
\end{theorem}
We defer the proof to \Cref{sec:proofs-stats}. Note that a law of large numbers for the Euler characteristic curve has been established in Corollary 6.2 of \citet{bobrowski2017vanishing} for all possible regimes and could be integrated to derive a similar result for hybrid transforms. It is here a key assumption that we are in the so-called \emph{sparse regime}, that is, $n r_n^d \to 0$. To make this law of large numbers more understandable, we make a further assumption that we are in the so-called \emph{divergence regime}, that is $n^{k+2}r_n^{d(k+1)} \to \infty$ for all $k \in\llbracket 0, d-1 \rrbracket$. The sequence defined by $r_n = n^{-\alpha}$ for $\frac{1}{d} < \alpha < \frac{1}{d}+\frac{1}{d^2}$ verifies these two assumptions. Similar results can be derived for other subcases of the sparse regime: the Poisson regime $ n^{k+2}r_n^{d(k+1)} \to c >0 $ and the vanishing regime $n^{k+2}r_n^{d(k+1)} \to 0$. 

\Cref{thm:LLN_HT1D} shows that the pointwise limit of the hybrid transform depends on the sampling only through the quantities $\int_{\mathbb{R}^d} g^{k+2}$ for $k = 0, 1, \ldots, d-1$ and they can therefore discriminate between different samplings as soon as $n$ is large enough. In addition to this law of large numbers, a finer limit result for the Euler characteristic curve is proven in \citet{krebs2021approximation}, which we recall hereafter for the sake of completeness. First, recall that a function $h$ on $\R^m$ is \emph{blocked} if it can be written $h = \sum_{i=1}^{m^d} b_i \1_{A_i}$ where~$b_1,\ldots, b_{m^d}$ are non-negative real numbers and the $A_i$ are axis-aligned rectangles in $\R^m$. Moreover, recall that the \emph{Skorohod $J_1$-topology} on the space of càdlàg functions~$D([0,T])$ is the topology induced by the metric:
\begin{equation*}
    d_{J_1}(f, g):=\inf _\lambda\left\{\|f \circ \lambda-g\|_{\infty}+\|\lambda-\id_{[0,T]}\|_{\infty}\right\},
\end{equation*} 
where the infimum is taken over all increasing continuous bijections of $[0, T]$.



\begin{theorem}[\textnormal{\bf \citealp[Theorem 3.4]{krebs2021approximation}}]
\label{TCL_ECC}
Let $T >0$ and $X_1, \ldots, X_n$ be sampled according to a bounded density $g$ on $[0, 1]^d$. Denote by $\filt_n$ the \Cech{} complex associated with the point cloud $n^{1/d} (X_i)_{i=1}^n$. Assume that blocked functions can uniformly approximate~$g$. There is a Gaussian process $\mathfrak{G} : [0, T] \to \R$ such that for $t \in [0, T]$, 
\begin{equation*}
    \sqrt{n}\big(\ECPn(t) - \mathbb{E}[\ECPn(t)]\big) \longtoinf \mathfrak{G}(t),
\end{equation*}
in distribution in the Skorohod $J_1$-topology. Furthermore, there exist two real-valued functions $\gamma$ and $\alpha$ such that the covariance of the limiting process is defined by:
\begin{equation*}
    \mathbb{E}[\mathfrak{G}(s) \mathfrak{G}(t)]=\mathbb{E}\left[\gamma\left(g(Z)^{1 / d}(s, t)\right)\right]-\mathbb{E}\left[\alpha\left(g(Z)^{1 / d} s\right)\right] \mathbb{E}\left[\alpha\left(g(Z)^{1 / d} t\right)\right],
\end{equation*}%
where $Z$ is a random variable with density $g$.

\end{theorem}

We refer to \citet{krebs2021approximation} for the expression of the two functions $\gamma$ and $\alpha$. Here again, the distribution of the points appears in the limiting object and, more precisely, in its covariance function. We can adapt this theorem to show that hybrid transforms of compactly supported kernels are also asymptotically normal.

\begin{theorem}
\label{thm:TCL_HT}
Consider the setting of \Cref{TCL_ECC}. Let $a, M > 0$ and~$\kernel\in\Lp{\R}$. Further assume that $\kernel$ is supported on $[0,T]$. Then, there is a Gaussian process $\tilde{\mathfrak{G}} : [a, M] \to \R$ such that:
\begin{equation*}
\sqrt{n} \left( \HTn - \mathbb{E}\left[ \HTn \right] \right) \ \longtoinf\  \tilde{\mathfrak{G}} \quad \text{a.s.},
\end{equation*}%
in $\big(\mathcal{C}^{0}[a, M], \| \cdot \|_{\infty}\big)$. 
Furthermore, the covariance of the limiting process is defined by:
\begin{equation*}
    \mathbb{E}\left[\tilde{\mathfrak{G}}(\xi_1) \tilde{\mathfrak{G}}(\xi_2)\right]= \xi_1 \xi_2 \int_0^{T/\xi_1} \int_0^{T/\xi_2} \kernel(\xi_1 t) \,\kernel (\xi_2 s) \,\cov\big(\mathfrak{G}(s),\mathfrak{G}(t)\big)\ds \dt,
\end{equation*}%
where $\mathfrak{G}$ is the Gaussian process defined in Theorem \ref{TCL_ECC}.

\end{theorem}







\subsection{Limit theorem for multi-parameter hybrid transforms}

Here, we adopt the sampling model of \citet{hiraoka2018limit}. Consider a point process~$\Phi$ on~$\R^d$ and its restriction $\Phi_L$ to $[-L/2, L/2]^d$. Let $\mathcal{S} (\mathbb{R}^d)$ be the collection of all finite (non-empty) subsets in $\mathbb{R}^d$, to be thought of as the set of all simplices. Let $f = (f_1, \ldots, f_m): \mathcal{S} (\mathbb{R}^d) \to [0, \infty]^m$ be a measurable function, non-decreasing with respect to the inclusions of faces. According to \Cref{ex:sublvl}, $f$ induces a filtration on every simplicial complex of $\mathbb{R}^d$. The following theorem derives a law of large numbers for the hybrid transform in the multi-parameter case.

\begin{theorem}
    \label{thm:LLN_multiD}
Assume that $\Phi$ is a stationary ergodic point process having finite moments. Let $T, a >0$ and $\kappa \in L^1(\mathbb{R})$. Assume that $\kernel$ is supported on $[0, T]$. We denote by $\filt_L$ the filtration induced by the sublevel sets of $f$ on~$\Phi_L$. Assume that there exists an increasing function $\rho$ such that there exists $i \in \llbracket 1, m \rrbracket$ such that for all $(x, y) \in (\mathbb{R}^d)^2$,
\begin{equation}
\label{eq:cond_filt}
\|x-y\| \leq \rho \left( f_i ( \{x, y \})  \right).
\end{equation}
Under these assumptions, there exists a function $H:\dual{\R_+^m}\times \mathbb{R}_+ \to\R$ that depends only on $\kappa$ and $f$ such that, for all $\xi = (\xi_1, \ldots, \xi_m) \in \dual{\R_+^m}$ and $\lambda >a$,
\begin{align*}
    \frac{1}{L^d}\HTL (\lambda \xi)  \ &\longtoinf[L \to \infty]\  H(\xi, \lambda) \quad \text{ a.s.}. 
\end{align*}
\end{theorem}


This limit theorem is a direct consequence of the results from \citet{hiraoka2018limit} for persistence diagrams of a large class of filtration functions. We refer to Section 3 of loc. cit. for the definition of a stationary ergodic point process. Note that this encompasses most cases of usual point processes such as Poisson, Ginibre, or Gibbs. This result makes use of the smoothness properties of the hybrid transforms and follows directly from \Cref{lem:HTn_to_HT1} that expresses restrictions of multi-parameter hybrid transforms to lines as one-parameter hybrid transforms. Similar results cannot be derived that easily for Euler characteristic profiles, as one would need to consider the joint law of several one-parameter filtrations. In addition, deriving a multi-dimensional central limit theorem from \citet{penrose2001central} would require the filter~$\dualdot{\xi}{f}$ to verify some translation invariance property. In practice, this strong assumption is verified by \Cech{} and Vietoris-Rips filtrations as well as marked processes; see \citet{botnan2022consistency}. However, function-\Cech{} and function alpha-filtrations used in our experiments do not verify this assumption in general. 

As pointed out in \citet[Example~1.3]{hiraoka2018limit}, \Cech{} and Vietoris-Rips filtrations satisfy~\eqref{eq:cond_filt} for $\rho : t \mapsto 2t$. We provide below two examples of families a broad family of multi-parameter filtrations satisfying~\eqref{eq:cond_filt}.
\begin{example}
    It is easy to check that the function-alpha filtration considered in the applications of \Cref{sec:orbit,sec:sidney} satifies~\eqref{eq:cond_filt}. 
\end{example}

We give another class of filtrations satisfying~\eqref{eq:cond_filt} that contains in particular the distance-to-measure (DTM) filtrations \citep{anai2020dtm}.
\begin{example}
\label{ex:weighted-Cech}
        Let $h$ be a positive and bounded function from $\mathbb{R}^d$ to $\R$. The weighted \Cech{} complex introduced in \citet{anai2020dtm} is defined as follows. For every $x \in \R^d$ and real number $t \geq 0$, we define:
        \begin{equation*}
            r_x(t)= 
            \begin{cases}
            -\infty & \text { if } t<h(x), \\
             t-h(x) & \text { otherwise.}
            \end{cases}
        \end{equation*}
    We denote by $\bar{B}_h(x, t)=\bar{B}\left(x, r_x(t)\right)$ the closed Euclidean ball of center $x$ and radius~$r_x(t)$. A simplex $\{x_0, \ldots, x_k\}$ in some finite set $\mathbb{X}$ belongs to the \textit{weighted \Cech{} complex} at scale $t\geq 0$ if the intersection of closed balls $\cap_{l=0}^k \bar{B}_h (x_l,t)$ is non-empty. Considering the weighted \Cech{} complex for all scales $t$ defines a filtration of $2^\X$ called \emph{weighted \Cech{} filtration}. 
    The weighted \Cech{} filtration satisfies~\eqref{eq:cond_filt} for $\rho : t \mapsto \max (\max h, 2t).$
\end{example}

\section{Proofs}
\label{sec:proofs}
In this section, we prove the results stated in \Cref{sec:stability,sec:stat}.

\subsection{Proofs of stability results}\label{sec:proofs-stability}
In the following proofs, we make constant use of the fact that the distance $\dist$ may be computed on any decomposition of the functions and not only on minimal ones, that is, on signed barcodes. More precisely, for any decompositions $(\mathcal{C}^+, \mathcal{C}^-)$ and $(\mathcal{C}'^+,\mathcal{C}'^-)$ of two finitely presented functions $\phi$ and $\phi'$ respectively, one has:
\begin{equation}\label{eq:signed-W1-with-decompostions}
    \dist\big(\phi, \phi'\big) = \inf \left\{ \eps>0 \st \exists \textnormal{ bijection } h:\mathcal{C}^+\cup\mathcal{C}'^- \to \mathcal{C}^-\cup\mathcal{C}'^+ \textnormal{ with } \cost(h) \leq \eps \right\}.
\end{equation}

\subsubsection{Profiles of finitely generated filtrations are finitely presented}\label{sec:proof-fg-filt-fp-ecp}
The following lemma is well-known. We prove it for completeness. Recall that the $k$-th \emph{Betti function} of a finitely generated filtration $\filt$ is defined as the function $\betti[\filt,k] : t\in\R^m \mapsto \dim H_k(\filt_t)$.
\begin{lemma}
    \label[lemma]{lem:ECP-is-FP}
    Let $\filt$ be a finitely generated $m$-parameter filtration. The $k$-th Betti function $\betti[\filt,k]$ is finitely presented and $\ECP = \sum_{k\in\Z} (-1)^k \betti[\filt,k]$. In particular, the Euler characteristic profile of $\filt$ is finitely presented.
\end{lemma}

\begin{proof}
    The fact that $\betti[\filt,k]$ is finitely presented follows from the fact that the family of vector spaces $(H_k(\filt_t))_{t\in\R^m}$ forms a \emph{finitely presented $m$-parameter persistence module} (see \Cref{lem:persistent-homology-finitely-presented}). This last fact is well known but goes beyond the scope of the paper and is not explicitly written elsewhere in the literature. We provide a proof in \nameref{sec:appendix}. The equality between the alternated sum of Betti functions of $\filt$ and its Euler characteristic profile follows from the classical formula for the Euler characteristic of any simplicial complex $\cplx$ stating that:
    \begin{equation*}
        \chi(\cplx) = \sum_{k\geq 0} (-1)^k \dim H_k(\cplx).
    \end{equation*}
    The fact that $\ECP$ is finitely presented is then straightforward.
\end{proof}


It is clear from the definition that the signed $1$-Wasserstein distance between Euler characteristic profiles is bounded from above by the same distance between Betti functions, as stated in the lemma below. It will be crucial to proving the other results.
\begin{lemma}\label[lemma]{lem:dist-ECP-dist-betti}
    Let $\filt$ and $\filt'$ be two finitely generated $m$-parameter filtrations of simplicial complexes~$\cplx$ and $\cplx'$ respectively. Then,
    \begin{equation*}
        \dist[\ECP, \ECP[\filt']] \leq \sum_{k\geq 0} \dist[\betti[\filt,k], \betti[\filt',k]].
    \end{equation*}
\end{lemma}

\subsubsection{Proof of \Cref{lem:dist-W1-barcode}}
The degree $k$ persistence diagram of $\filt$ is given by ${\diagram_k} = \left\{\big(a^k_i,b^k_i\big)\right\}_{i=1,...,n_k}$ for real numbers $-\infty < a^k_i < b^k_i \leq \infty$ and an integer $n_k\geq 0$. This diagram induces a decomposition $(\mathcal{A}_k,\mathcal{B}_k)=(\{a_i^k\}_i, \{b_i^k\}_i)$ of $\betti[\filt,k]$. Similarly, the degree $k$ persistence diagram ${\diagram'_k} = \left\{\big(a_j'^k,b_j'^k\big)\right\}_{j=1,...,n'_k}$ of $\filt'$ induces a decompositon $(\mathcal{A}_k',\mathcal{B}_k')=(\{a_j'^k\}_j, \{b_j'^k\}_j)$ of $\betti[\filt',k]$. Moreover, a partial matching $\matching$ between ${\diagram_k}$ and ${\diagram_k'}$ induces a bijection of multisets~$h:\mathcal{A}_k'\cup\mathcal{B}_k \to \mathcal{A}_k\cup\mathcal{B}'_k$ defined by $h(a') = a$ and $h(b)=b'$ when $((a,b),(a',b'))\in\matching$, by $h(b)=a$ when $(a,b)$ is unmatched and by $h(a') = b'$ when $(a',b')$ is unmatched. Moreover, the cost of the matching $\matching$ and the cost of the bijection $h$ satisfy $\cost(h) \leq 2 \,\cost(M)$. Taking the infimum over all partial matching $\matching$, one has~$\dist[\betti[\filt,k], \betti[\filt',k]]\leq W_1(H_k\filt, H_k\filt')$. \Cref{lem:dist-ECP-dist-betti} yields the result.

\subsubsection{Proof of \Cref{prop:stability-ECP}}
Consider decompositions $(\barcode^+, \barcode^-)$ and $(\barcode'^+, \barcode'^-)$ of $\ECP[\filt]$ and $\ECP[\filt']$ respectively. Assume there is a bijection $h:\barcode^+\cup\barcode'^- \to \barcode^-\cup\barcode'^+$. If no such bijection exists, then~$\dist[\ECP[\filt], \ECP[\filt']]$ is infinite, and the inequality trivially holds. One has:
\begin{equation*}
    \ECP[\filt] - \ECP[\filt'] 
    = \sum_{u\in\barcode^+\cup\barcode'^-} \1_{Q_u} - \sum_{v\in\barcode^-\cup\barcode'^+} \1_{Q_v} 
    = \sum_{u\in\barcode^+\cup\barcode'^-} \1_{Q_u} - \1_{Q_{h(u)}}.
\end{equation*}
Therefore, recalling the definition of $\|-\|_{1,M}$ from \Cref{sec:definitions}, we have:
\begin{equation}\label{eq:L1-bdd-by-matching}
    \|\ECP[\filt] - \ECP[\filt']\|_{1,M} \leq \sum_{u\in\barcode^+\cup\barcode'^-} \|\1_{Q_u} - \1_{Q_{h(u)}}\|_1.
\end{equation}
By an elementary induction on~$m\geq 1$, we can prove that for all $u, v \in \R^m$,
\begin{equation*}
\|\1_{Q_u}-\1_{Q_v}\|_{1,M} \leq (2M)^{m-1} \|u-v\|_1.
\end{equation*}
This conludes the proof.

Assume now that $m=1$. The existence of $h$ ensures that $\|\ECP[\filt] - \ECP[\filt']\|_{1}$ is finite and the result follows from~\eqref{eq:L1-bdd-by-matching} and the fact that $\|\1_{[u,v)}\|_1 = |u-v|$.

\subsubsection{Proof of \Cref{cor:stability-Rdn-ht}}
It follows from the definition of hybrid transforms that:
\begin{equation*}
    \|\HT[\filt] - \HT[\filt']\|_{L^q_K} \leq 
    \begin{cases}
        \displaystyle \|\kernel\|_\infty\int_K\int_\R |\xi_*\ECP(s) - \xi_*\ECP[\filt'](s)| \d s \d \xi &\mbox{if } q\in[1,\infty), \\[1em]
        \displaystyle \|\kernel\|_\infty\sup_{\xi\in K} \int_\R |\xi_*\ECP(s) - \xi_*\ECP[\filt'](s)| \d s &\mbox{if } q = \infty.
    \end{cases}
\end{equation*}
Moreover, \Cref{prop:stability-ECP} with $m=1$ ensures that for any~$\xi\in K$,
\begin{equation*}
    \|\xi_*\ECP[\filt]-\xi_*\ECP[\filt']\|_1 \leq \dist[\xi_*\ECP[\filt], \xi_*\ECP[\filt']].
\end{equation*}
To prove the desired inequality, we will prove that $\dist[\xi_*\ECP[\filt], \xi_*\ECP[\filt']] \leq \|\xi\|_\infty\, \dist[\ECP[\filt], \ECP[\filt']]$ for any $\xi\in\dual{\R_+^m}$. The result then follows from computing the $q$-norm on both sides. Consider decompositions $(\barcode^+, \barcode^-)$ and $(\barcode'^+, \barcode'^-)$ of~$\ECP[\filt]$ and~$\ECP[\filt']$ respectively. They induce decompositions $(\xi_*\barcode^+, \xi_*\barcode^-)$ and $(\xi_*\barcode'^+, \xi_*\barcode'^-)$ of~$\xi_*\ECP[\filt] = \ECP[\xi_*\filt]$ and~$\xi_*\ECP[\filt']=\ECP[\xi_*\filt']$ respectively by the formula $\xi_*\barcode^\pm = \{\dualdot{\xi}{u} \st u\in\barcode^\pm\}$ and a similar one for $\filt'$. Consider a bijection of multisets~$h:\barcode^+\cup\barcode'^- \to \barcode^-\cup\barcode'^+$. It induces a bijection of multisets $\xi_*h:\xi_*\barcode^+\cup\xi_*\barcode'^- \to \xi_*\barcode^-\cup\xi_*\barcode'^+$ defined by $\dualdot{\xi}{u} \mapsto \dualdot{\xi}{h(u)}$ with cost:
\begin{equation*}
    \cost(\xi_*h) = \sum_{t\in\xi_*\barcode^+\cup\xi_*\barcode'^-} \|t - \xi_*h(t)\|_1 = \sum_{u\in\barcode^+\cup\barcode'^-} \|\dualdot{\xi}{u} - \dualdot{\xi}{h(u)}\|_1 \leq \|\xi\|_\infty \cdot \cost(h).
\end{equation*}
Taking the infimum over all bijections $h$ yields $\dist[\xi_*\ECP[\filt], \xi_*\ECP[\filt']] \leq \|\xi\|_\infty \,\dist[\ECP[\filt], \ECP[\filt']]$ by~\eqref{eq:signed-W1-with-decompostions}.

\subsubsection{Proof of \Cref{lem:stability-L1-filt-fn}}
The couple $\mathcal{C}_f = \big(\{f(\sigma)\}_{\dim\sigma \text{ even}}, \{f(\sigma)\}_{\dim\sigma \text{ odd}}\big)$ is a decomposition of $\ECP[f]$. There is a similar decomposition $\mathcal{C}_g$ of $\ECP[g]$. Moreover, the mapping $f(\sigma)\mapsto g(\sigma)$ induces a bijection of multisets $h:\mathcal{C}_f\to \mathcal{C}_g$ with cost $\cost(h) = \sum_{\sigma\in\cplx} \|f(\sigma)-g(\sigma)\|_1 = \|g-f\|_1$. The result follows from~\eqref{eq:signed-W1-with-decompostions}.

\subsection{Proofs of statistical results}\label{sec:proofs-stats}
In this section, we prove the asymptotic results for the hybrid transforms stated in \Cref{sec:stat}.
\subsubsection{Proof of Theorem \ref{thm:LLN_HT1D}}
Let $X_1, \ldots, X_n$ be an i.i.d. sample drawn according to an a.e. continuous bounded Lipschitz density $g$ on $\mathbb{R}^d$. Consider a sequence~$(r_n)_{n\in\N}$ such that $n r_n^d \to 0$ and~$n^{k+2}r_n^{d(k+1)} \to \infty$ as $n\to\infty$. 

Let us define $\Delta:=\{(x, y): 0 \leq x \leq y<\infty\} \cup\{(x, \infty): 0 \leq x<\infty\}$ and for every~$(s, t, u, v)$ such that $0 \leq s \leq t \leq u \leq v \leq \infty$, denote by $R_{s, t, u, v}$ the rectangle~$(s, t] \times (u, v]$ of $\Delta$. Recall that a finite persistence diagram $\diagram = \cup_{i=1}^l(a_i, b_i)$ can be turned into a discrete measure~$\mu = \sum_{i=1}^l \delta_{a_i, b_i}$ on $\Delta$. Denote by $\mu_{k,n}$ the $k$-th persistence diagram of the \Cech{} filtration of~$1/r_n (X_i)_{i=1}^n$, seen as a discrete measure on $\Delta$. 

Theorem 3.2 of \citet{owada2022convergence} ensures that for every $k \in \llbracket 0, d-1 \rrbracket$ there exists a unique Radon measure $\mu_k$ on $\Delta$ such that we have the following vague convergence:
\begin{equation}\label{eq:vague-convergence}
\frac{1}{n^{k+2} r_n^{d(k+1)}}\, \mu_{k, n} \ \stackrel{v}{\longtoinf}\  \frac{1}{(k+2)!} \left(\int_{\R^d} g^{k+2}(x)\d x\right) \mu_k \quad \text{ a.s.},
\end{equation}
where for every $0 \leq s \leq t \leq u \leq v \leq \infty$, there is an indicator geometric function~$H_{s, t, u, v}$ on $\R^{d(k+2)}$ defined in \cite[Sec.~3.1]{owada2022convergence}, which does not depend on $g$ and such that the measure $\mu_k$ is defined by:
\begin{equation*}
    \mu_k (R_{s, t, u, v}) =  \int_{\mathbb{R}^{d(k+1)}} H_{s, t, u, v} (0, y_1, \ldots, y_{k+1}) \d y_1\ldots\d y_{k+1}.
\end{equation*}
Recall that the primitive kernel $\Kernel$ is such that $\Kernel(x) \to 0$ when $x\to+\infty$. Therefore, the fact that $\kernel$ is supported on $[0, T]$ implies that the primitive $\Kernel$ is also supported on~$[0, T]$. For $\xi >a$, denote by $h_{\xi} : (x, y)\in\Delta \mapsto \Kernel (\xi y) - \Kernel(\xi x)$. According to~\eqref{eq:PD_HT}, one has:
\begin{equation*}
\HTn (\xi) = \sum_{k=0}^{d-1} (-1)^k\langle \mu_{k,n}, h_{\xi} \rangle.
\end{equation*}
Since $h_{\xi}$ is continuous and supported on $[0, T/a]^2$, we have by the vague convergence in~\eqref{eq:vague-convergence} that:
\begin{equation*}
    \frac{1}{n^{k+2}r_n^{d(k+1)}}\, \HTn(\xi) \ \longtoinf\  \sum_{k=0}^{d-1} \frac{(-1)^k}{(k+2)!} \left(\int_{\mathbb{R}^d} g^{k+2} (x) \mathrm{d}x\right) A_k (\xi) \quad \text{ a.s.},    
\end{equation*}
where $A_k (\xi) = \int_{\Delta} h_{\xi} \mathrm{d} \mu_k$.

\subsubsection{Proof of Theorem \ref{thm:TCL_HT}}

Let $T >0$ such that $\kappa$ is supported in $[0, T]$. Let $a, M>0$ and let $\xi \in [a, M]$. According to~\eqref{eq:HT-as-classical}, we have that:
\begin{equation*}
    \HT (\xi) = \xi \int_0^{T/\xi} \kernel( \xi \cdot t) \ECP(t) \mathrm{d}t,  
\end{equation*}
and similarly for $\ECPn$. Since $\kernel$ is in $L^1$, the mappings $\HT$ and $\HTn$ are continuous on $[a, M]$. According to Theorem \ref{TCL_ECC}, there is a Gaussian process $\mathfrak{G} : [0, T/a] \to \mathbb{R}_{+}$ such that for all $t \in [0, T/a]$, we have that:
\begin{equation}
\label{eq:TCL_chi}
    \sqrt{n}\, \big(\ECPn(t) - \mathbb{E}[\ECPn(t)]\big) \ \longtoinf\  \mathfrak{G}(t),
\end{equation}
in distribution in the Skorohod $J_1$-topology. Therefore, by linearity of the mapping~$\chi \mapsto \xi \int_0^{T/\xi} \kernel(\xi \cdot t) \chi(t) \mathrm{d}t$, we have that:
\begin{align*}
    \sqrt{n} \left( \HTn - \mathbb{E}\left[ \HTn \right] \right) \ =\ \,&\  \xi \int_0^{T/ \xi} \kernel(\dualdot{\xi}{t}) \left[ \sqrt{n} \big(\ECPn(t)-\mathbb{E}[\ECPn(t)]\big) \right] \dt
\end{align*}
Denote by $\phi$ the mapping from the space of càdlàg functions $D([0, T])$ with Skorohod $J_1$-topology to $(\mathcal{C}^0([a, M]), \|\cdot\|_\infty)$ defined by:
\begin{equation*}
    \phi: \chi \mapsto \left( \xi \mapsto \xi \int_0^{T/\xi} \kernel(\xi \cdot t) \chi(t) \mathrm{d}t  \right).    
\end{equation*}
We, therefore, have that:
\begin{equation*}
    \sqrt{n} \left( \HTn - \mathbb{E}\left[ \HTn \right] \right) \ =\ \phi \left( \sqrt{n} \left(\ECPn-\mathbb{E}[\ECPn] \right) \right).
\end{equation*}
It is easy to check that:
\begin{equation*}
    \| \phi (\chi_1)- \phi (\chi_2) \|_\infty \leq \frac{M}{a}\, \|\chi_1 -\chi_2\|_{\infty} \int_0^T |\kernel(u)| \mathrm{d}u, 
\end{equation*}
so that the mapping $\phi$ is Lipschitz and, therefore, continuous. Thus, the continuous mapping theorem along with~\eqref{eq:TCL_chi} yields that almost surely, one has the following convergence in $(\mathcal{C}^0([a, M]), \|\cdot\|_\infty)$,
\begin{equation*}
    \sqrt{n} \left( \HTn - \mathbb{E}\left[ \HTn \right] \right) \  \longtoinf \  \tilde{\mathfrak{G}}(\xi) := \xi \int_0^{T/ \xi} \kernel(\dualdot{\xi}{t}) \mathfrak{G}(t) \dt.
\end{equation*}
The covariance of the limiting process $\tilde{\mathfrak{G}}$ follows immediately from that of $\mathfrak{G}$.

\subsubsection{Proof of Theorem \ref{thm:LLN_multiD}}
Let $\xi = (\xi_1, \ldots, \xi_m) \in \dual{\R_+^m}$. Denote by $\mu_{k, L}^{\xi_*\filt}$ the measure associated with the $k$-th persistence diagram of $\Phi_L$ for the filtration function $\dualdot{\xi}{f} = \sum_{i=1}^m \xi_i f_i$. By hypothesis, there exists $i \in \llbracket 1, m \rrbracket$ such that for all $(x, y) \in (\mathbb{R}^d)^2$,
$ \|x-y\| \leq \rho \left( f_i ( \{x, y \})  \right)$. Let $\rho^\prime: x \mapsto \rho(x/\xi_i)$. Therefore, as the filtration functions are non-negative and $\rho$ and $\rho^\prime$ are increasing, we have that:
\begin{equation}
\label{eq:hyp_hiraoka}
\rho^\prime \left(\sum_{j=1}^m \xi_j f_j(\{x, y\})\right) \geq \rho^\prime( \xi_i f_i(\{x,y\}) \geq \rho (f_i(\{x, y\})) \geq \|x-y\|.
\end{equation}
The filtration function $\dualdot{\xi}{f}$ therefore verifies all the hypotheses of Theorem 1.5 of \citet{hiraoka2018limit}, which states that there exists a Radon measure $\nu_k$ such that almost surely, we have the vague convergence $\frac{1}{L^d}\mu_{k, L}^{\xi_*\filt}\overset{v}{\to}\nu_k^{\dualdot{\xi}{f}}$ as $L\to\infty$. Note that in loc. cit., the authors make the additional hypothesis that the filtration function is translation invariant. However, this assumption is only needed to derive a central limit theorem on persistent Betti numbers but not required for the above law of large numbers, for which we only need~\eqref{eq:hyp_hiraoka} to hold. As in the proof of Theorem \ref{thm:LLN_HT1D}, we introduce a continuous function $h_{\lambda}: (x, y)\in\Delta \mapsto \Kernel (\lambda y)-\Kernel(\lambda x)$. This function is supported on $[0, T/a]^2$. According to~\eqref{eq:PD_HT} together with \Cref{lem:HTn_to_HT1}, we have that:
\begin{equation*}
\HTL (\lambda \xi) = \sum_{k=0}^{d-1}(-1)^k \langle \mu_{k, L}^{\xi_*\filt}, h_{\lambda} \rangle.
\end{equation*}
Hence the result, by the vague convergence $\frac{1}{L^d}\mu_{k, L}^{\xi_*\filt} \overset{v}{\to}\nu_k^{\dualdot{\xi}{f}}$ for every $k \in \llbracket 0, d-1 \rrbracket$.

\acks{The authors are grateful to Steve Oudot, François Petit, Clément Levrard, Wolfgang Polonik and Mathieu Carrière for useful discussions. The authors are also grateful to the Inria Datashape team for comments on an early version of this work at the annual seminar and to the anonymous reviewer for their insightful comments that helped greatly improve the exposition. The second named author was funded in part by EPSRC EP/R018472/1. For the purpose of Open Access, the author has applied a CC BY public copyright licence to any Author Accepted Manuscript (AAM)
version arising from this submission.}


\appendix

\section*{Appendix A}
\label{sec:appendix}
\renewcommand{\thesection}{A}

\setcounter{section}{0}
\setcounter{theorem}{0}
\setcounter{equation}{0}
\setcounter{figure}{0}
\setcounter{table}{0}
\renewcommand{\thetheorem}{\Alph{theorem}}

In this appendix, we prove that a finitely generated filtration has finitely presentable persistent homology. As explained in \Cref{sec:proofs-stability}, this fact is well-known. Its proof is included for completeness. We follow the same notations and conventions as in \citet[Section~2]{oudot2021stability}.
\begin{lemma}
    \label[lemma]{lem:persistent-homology-finitely-presented}
    Let $\filt$ be a finitely generated $m$-parameter filtration of a simplicial complex $\cplx$ and let~$k\geq 0$. The $m$-parameter persistence module $H_k(\filt)$ is finitely presentable. In particular, its Hilbert function is finitely presented.
\end{lemma}
\begin{proof}
    Since $\filt$ is finitely generated, the support of any $\sigma\in\cplx$ has a finite number of minimal elements. The set of these elements is called the \emph{births} of $\sigma$ and denoted by $\birth(\sigma)$. Since $\cplx$ is finite and $\filt$ is finitely generated, there is a finite subset~$G = I_1 \times \ldots \times I_m \subseteq \R^m$ such that $\birth(\sigma)\subseteq G$ for any $\sigma\in\cplx$. 
    
    Given a persistence module~$M$ over~$\R^m$, we denote by $r(M)$ its restriction to $G$. Given a persistence module~$N$ over~$G$, the \emph{extension of $N$} is the persistence module $e(N)$ over $\R^m$ defined by:
    \begin{equation*}
        e(N)(t) = N\left(\max\{g\in G \st g \leq t\}\right).
    \end{equation*}
    This defines functors $r$ and $e$ between the category of persistence modules over $\R^m$ to the category of persistence modules over $G$ and conversely. It is an easy exercise to check that these functors are exact. 
    
    We prove that $H_k(\filt) \simeq e\circ r (H_k(\filt))$. It is well known---see \citet[Lemma~5]{bauer2022generic} for a proof---that this implies that $H_k(\filt)$ is finitely presentable. Recall that \emph{barcodes} of free persistence modules are defined in \citet[Section~2]{oudot2021stability} and denote by $\chains_{i}(\filt)$ the free persistence module with barcode $\bigcup_{\sigma}\birth(\sigma)$ where the union is taken over all simplices $\sigma\in\cplx$ of dimension $i$. Consider the diagram:
    \begin{equation*}
        \begin{tikzcd}
        \chains_{k+1}(\filt) \arrow[r, "\partial_k"] & \chains_{k}(\filt) \arrow[r, "\partial_{k-1}"] & \chains_{k-1}(\filt),
        \end{tikzcd}
    \end{equation*}
    where the maps $\partial_k$ and $\partial_{k-1}$ are induced by the boundary operator from simplicial homology. The persistence module $H_k(\filt)$ is then the homology of the above diagram, i.e.,
    \begin{equation*}
        H_k(\filt) = \Ima(\partial_k)/\Ker(\partial_{k-1}).
    \end{equation*}
    For any $i\in\N$, the definition of $G$ and of $\chains_i(\filt)$ implies that $\chains_i(\filt) \simeq e\circ r(\chains_i(\filt))$. The result then follows from the fact that $e\circ r$ is an exact functor and hence commutes with computing homology. 

    \medskip

    We are left to prove that the Hilbert function of $H_k(\filt)$ is finitely presented. Since the persistence module $H_k(\filt)$ is finitely presentable, it admits a finite free resolution:
    \begin{equation}\label{eq:free-resolution}
        0 \to F_m \to \dots \to F_0 \to  H_k(\filt) \to 0.
    \end{equation}
    See for instance \citet[Section~7.2]{botnan2022introduction} for more details. Each free module with barcode $\barcode(F_i)$ has a finitely presented Hilbert function:
    \begin{equation*}
        \Hil(F_i) \, = \sum_{t\in\barcode(F_i)} \1_{Q_t}.
    \end{equation*}
    Now, exactness of the sequence~\eqref{eq:free-resolution} ensures that:
    \begin{equation*}
        \Hil(H_k(\filt)) = \sum_{i=0}^m (-1)^i \Hil(F_i),
    \end{equation*}
    hence the result.
\end{proof}


\bibliography{biblio}

\end{document}